\documentclass[12pt]{alt2020} 

\usepackage{natbib}
\setcitestyle{authoryear,open={(},close={)}}

\newcommand{\op}{\begin{itemize}}
\newcommand{\ed}{\end{itemize}}

\newcommand{\opq}{\begin{quote}}
\newcommand{\edq}{\end{quote}}

\newcommand{\ope}{\begin{enumerate}}
\newcommand{\ede}{\end{enumerate}}

\newcommand{\im}{\item}

\newcommand{\uu}{\mathbf{u}}
\newcommand{\vv}{\mathbf{v}}
\newcommand{\ww}{\mathbf{w}}
\newcommand{\xx}{\mathbf{x}}

\newcommand{\VV}{\mathbf{V}}
\newcommand{\UU}{\mathbf{U}}
\newcommand{\WW}{\mathbf{W}}

\newcommand{\Data}{\mathsf{Data}}

\newcommand{\II}{\mathcal{I}}

\newcommand{\lee}{\mathrm{left}}
\newcommand{\rii}{\mathrm{right}}

\newcommand{\Prob}{\mathbb{P}}
\newcommand{\Exp}{\mathbb{E}}

\newcommand{\VVV}{\mathcal{V}}

\newcommand{\SSS}{\mathcal{S}}
\newcommand{\DDD}{\mathcal{D}}

\newcommand{\HH}{\mathcal{H}}

\newcommand{\XX}{\mathbf{X}}

\newcommand{\PPP}{\mbox{$ \left( \VVV, \SSS, \HH \right) $}}

\newcommand{\given}{\mid}

\newcommand{\indep}{\:\raisebox{0.05em}{\rotatebox[origin=c]{90}{$\models$}}\:}

\newtheorem*{observation*}{Observation}

\title[Causal Learning Without Faithfulness]{On Learning Causal Structures from Non-Experimental Data \\without Any Faithfulness Assumption}
\usepackage{times}

\altauthor{%
 \Name{Hanti Lin} \Email{ika@ucdavis.edu}\\
 \addr Philosophy Department, University of California, Davis
 \AND
 \Name{Jiji Zhang} \Email{jijizhang@ln.edu.hk}\\
 \addr Philosophy Department, Lingnan University
}

\begin{document}

\maketitle

\begin{abstract}
Consider the problem of learning, from non-experimental data, the causal (Markov equivalence) structure of the true, unknown causal Bayesian network (CBN) on a given, fixed set of (categorical) variables. This learning problem is known to be very hard, so much so that there is no learning algorithm that converges to the truth for all possible CBNs (on the given set of variables). So the convergence property has to be sacrificed for some CBNs---but for which? In response, the standard practice has been to design and employ learning algorithms that secure the convergence property for at least all the CBNs that satisfy the famous \emph{faithfulness} condition, which implies sacrificing the convergence property for some CBNs that violate the faithfulness condition (Spirtes, Glymour, and Scheines, 2000). This standard design practice can be justified by assuming---that is, accepting on faith---that the true, unknown CBN satisfies the faithfulness condition. But the real question is this: Is it possible to explain, \emph{without assuming} the faithfulness condition or any of its weaker variants, why it is mandatory rather than optional to follow the standard design practice? This paper aims to answer the above question in the affirmative. We first define an array of modes of convergence to the truth as desiderata that might or might not be achieved by a causal learning algorithm. Those modes of convergence concern (i) how pervasive the domain of convergence is on the space of all possible CBNs and (ii) how uniformly the convergence happens. Then we prove a result to the following effect: for \emph{any} learning algorithm that tackles the causal learning problem in question, if it achieves the best achievable mode of convergence (considered in this paper), then it \emph{must} follow the standard design practice of converging to the truth for at least all CBNs that satisfy the faithfulness condition---it is a requirement, not an option.
\end{abstract}


\section{Introduction}

Suppose that there is a causal system that can be properly modeled by some causal Bayesian network (CBN) on a given set of observable variables, and that we aim to learn the causal structure of the true, unknown CBN (at least up to Markov equivalence), which is crucial to predicting what causal effects there would be if we were to manipulate this or that variable.
Suppose, further, that we wish to learn the causal structure only from non-experimental data, possibly because experimentation is too costly or unethical. It is well-known that this learning problem is very hard. The difficulty is that there can be two very different CBNs that are indistinguishable in terms of non-experimental data. To be more precise, there can be two CBNs $N$ and $N'$ with the following properties: 
	\ope  
	\im ({\sc Causal Difference}) $N$ and $N'$ have quite {\em different} causal structures that are not even Markov equivalent. 
	\im ({\sc Statistical Nonidentifiability}) $N$ and $N'$ share the {\em same} joint probability distribution; so, if a learning algorithm receives only non-experimental data (i.e., data collected without causing a change in the joint distribution), then it must, at any sample size, fail to have a high probability of identifying the true structure either for $N$ or for $N'$.
	\ede  
Because of property 1 (causal difference), it would be great if we could have a learning algorithm that converges in probability to the true causal structure (up to Markov equivalence) for both $N$ and $N'$ and, hopefully, for all CBNs on the given set of variables. But, unfortunately, no learning algorithm can be that good, by property 2 (statistical nonidentifiability). So, when we design a causal learning algorithm (also called causal discovery algorithm), the convergence property {\em must be sacrificed} for $N$ or for $N'$ and similarly for any other pair of CBNs to the same effect. Sacrifices have to be made for some---but for which? 


In reaction to that difficulty, the standard practice has been to design and employ learning algorithms that secure the convergence property for {\em at least} all the CBNs that satisfy the famous {\em faithfulness} condition \citep*{spirtes2000causation}, which implies sacrificing the convergence property for some (possibly not all) CBNs that violate the faithfulness condition. Examples abound, including constraint-based algorithms such as {\em PC} \citep{spirtes2000causation}, score-based algorithms such as {\em GES} \citep{chickering2002optimal}, and hybrids of those two kinds of algorithms \citep{zhalama2017weakening}. This standard design practice can be justified if we are willing to simply assume---that is, accept on faith---that the unknown, true CBN turns out to satisfy the faithfulness condition. But the real question is this: {\em Can we justify this standard design practice without assuming the faithfulness condition or any of its weaker variants?} This is the question that this paper aims to address---and answer in the affirmative. 

While we believe that the question just posed is very important, there is only a very small literature that attempts to address it. As far as we know, very few works try to address the issue explicitly; two notable ones are \cite{spirtes2000causation} and \cite{meek1995strong}. They show that the faithfulness condition only rules out a mathematically negligible set of CBNs (in the sense of negligibility that, roughly, a lower-dimensional plane is negligible in a higher-dimensional space). So any learning algorithm that follows the standard design practice sacrifices the convergence property only for a negligible set of CBNs. So it seems that we should not worry too much about using such a learning algorithm. 

But some question remains to be addressed. To be sure, we should not worry about using a learning method that sacrifices the convergence property for some CBNs that form a mathematically negligible set, precisely because no learning method can avoid sacrificing that much. So sacrifices have to be made at least for {\em some} mathematically negligible set of CBNs---but for {\em which} should sacrifices be made? One option is to follow the standard design practice. But there are alternatives, such as this one: (i) identify two CBNs $N$ and $N'$ that share the same joint distribution, with $N$ satisfying the faithfulness condition and $N'$ violating that condition, (ii) design and adopt a learning algorithm whose domain of convergence to the truth is the same as the set of faithful CBNs except that the faithful one $N$ is removed from the domain of convergence and the unfaithful one $N'$ is added to it. Such a learning algorithm is one of the infinitely many alternatives that sacrifice the convergence property only for a mathematically negligible set of CBNs but run counter to the standard design practice, which tries to secure the convergence property for at least {\em all} faithful CBNs. {\em But should we follow the standard design practice rather than any of those competing alternatives? If so, why?} That's the question.

To answer the above question, this paper develops a general, straightforward strategy:
	\begin{quote}
	{\sc General Strategy.} When the learning problem in question is extremely hard, so much so that (almost) every familiar desideratum for learning algorithms is provably too high an ideal to be achievable, we do not have to react by making an assumption that turns the learning problem into an easier one. Instead, we can react in this way: keep the learning problem as it is, look for what can be achieved, and determine what it takes for a learning method to achieve the highest achievable desideratum.
	\end{quote}
This strategy is implemented by defining certain modes of convergence to the truth that can be taken as desiderata for learning algorithms. To begin with, consider the question of where convergence happens. It would be great to extend the domain of convergence to cover everywhere on the space of all CBNs (on the given set of variables); but that is provably impossible for causal learning. So we examine the possibility of achieving some lower ideals (and their combinations): 
	\op 
	\im[$(a)$] extending the domain of convergence to cover {\em almost everywhere} on the space of all CBNs, i.e., everywhere except on a topologically negligible subset (a nowhere dense subset);
	
	\im[$(b)$] having a {\em maximal} domain of convergence, i.e., one that cannot be extended further.
	\ed 
This leads to the modes of convergence listed on the axis in figure \ref{fig-modes} that stretches to the upper right. The other axis, which stretches to the upper left, concerns the question of how uniformly convergence happens. 
	\begin{figure}[ht]
	\centering	\includegraphics[width=0.55\textwidth]{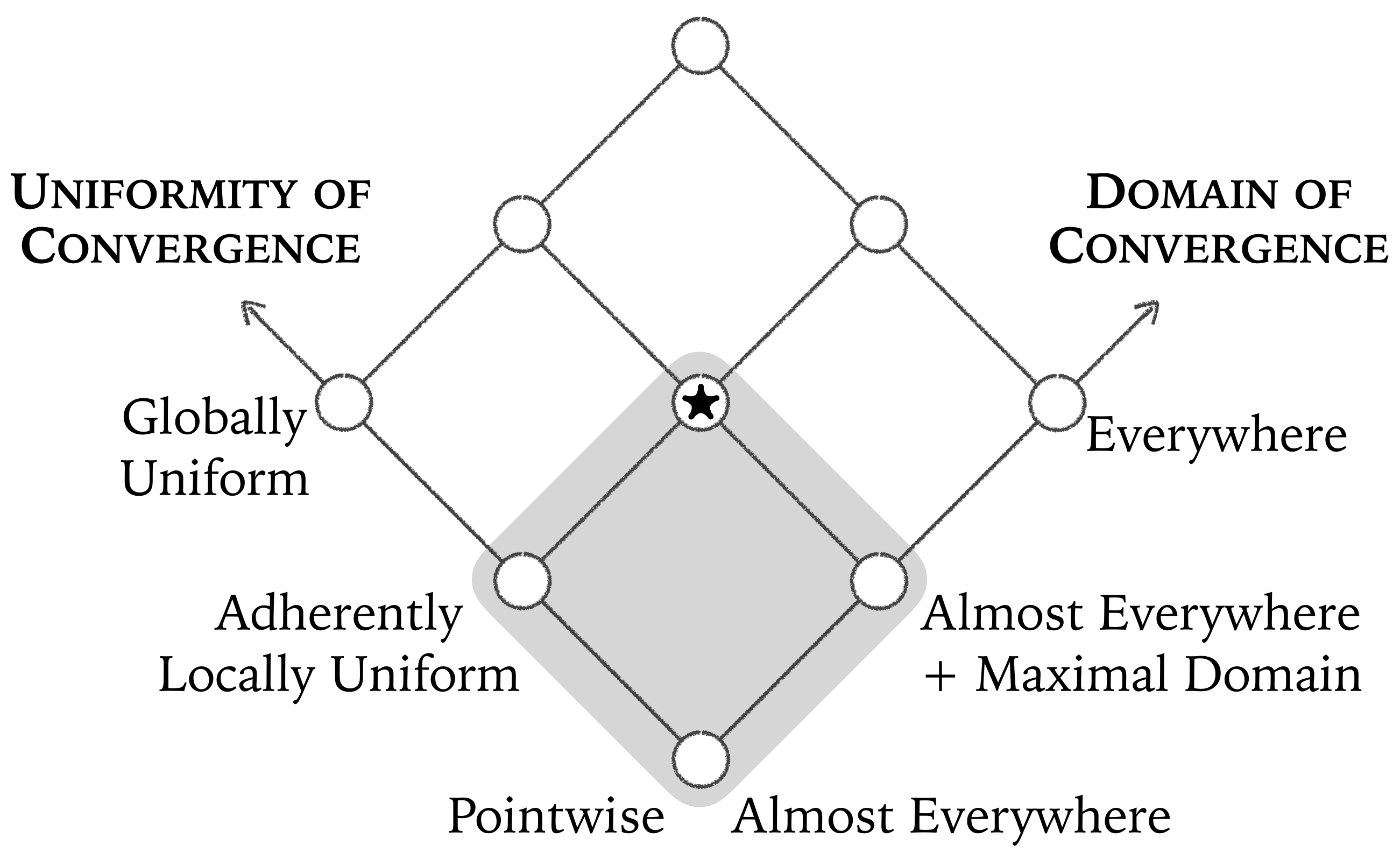}
	\caption{Modes of convergence to the truth}
	\label{fig-modes}
	\end{figure}
It would be great to achieve globally uniform convergence (aka uniform consistency), but that is provably impossible for causal learning. That is, no causal learning algorithm can guarantee a bound on the error probability that applies to all CBNs. So we examine the possibility of achieving something local rather than global:  
    \op 
    \im[$(c)$] {\em locally uniform} convergence of a certain kind---the ``adherent'' kind---that guarantees that a low error probability can be obtained stably under small perturbations of the joint probability distribution but without a change in the causal structure (that is, with ``adherence'' to the unchanged causal structure).
    \ed 
Those two considerations---about domain of convergence and uniformity of convergence---are then used to define nine joint modes of convergence, depicted as the nine nodes in the lattice in figure \ref{fig-modes}. (Some of those joint modes are equivalent because globally uniform convergence, alone, is strong enough to imply anything else in the figure.) The main result of this paper is theorem \ref{the:general1} (stated in section \ref{sec:result1}), which can be summarized as follows:
	\op
	\im Of the joint modes of convergence in figure \ref{fig-modes}, the achievable ones for the problem of learning the true causal structure (up Markov equivalence) are exactly those in the shaded area. So the best achievable one is the one marked by a star $\star$ in the figure, which conjoins the three modes of convergence $(a)$ ``almost everywhere'', $(b)$ ``maximal domain'', and $(c)$ ``adherently locally uniform''. 
	\im For {\em any} causal learning algorithm $L$, if $L$ satisfies at least $(a)$, $(b)$, and $(c)$ simultaneously (whether or not it satisfies any additional desiderata about, say, rates of convergence or computational complexity), then $L$ {\em must} follow the standard design practice in that it converges to the truth for at least all CBNs that satisfy the faithfulness condition---this is a requirement, not an option.
	\ed 


To the best of our knowledge, this is the first theoretical result that explains, without assuming the faithfulness condition or any of its weaker variants, why it is mandatory rather than optional to follow the standard design practice. This result is proved for any fixed finite set of categorical variables, under just the standard assumption of IID (identically and independently distributed observations) and the assumptions built into the definition of causal Bayesian networks.

This paper actually does more. To achieve at least the desideratum of $(a)$+$(b)$+$(c)$, convergence to the truth {\em must be secured} for some range of CBNs, {\em must be sacrificed} for some other range, and is {\em optional} for the remaining range. Those three ranges are precisely determined in a strengthening of the main result, theorem \ref{the:general2} (stated in section \ref{sec:result2}). 

The rest of this paper proceeds as follows. Standard definitions are reviewed in sections~\ref{sec:setup1} with examples. Key definitions are provided and motivated in section~\ref{sec:setup2}. The main result is stated and discussed in section \ref{sec:result1}, followed by a strengthened result in section \ref{sec:result2}. Section \ref{sec:proof-crucial-lemma} provides an illustrated proof of a quite revealing lemma. Complete proofs are left to the appendix. To declare the style in use: Emphasis is indicated by {\em italics}; the terms to be defined are presented in \textbf{boldface}.

\section{Review of Standard Definitions}\label{sec:setup1}


Fix a finite set of variables, $\VVV = \{X_1, X_2, \ldots, X_K\}$. A possible \textbf{causal structure} over those variables is represented by a directed acyclic graph on $\VVV$, written $G = (\VVV, \to)$, where the binary relation $X_i \to X_j$ is understood to say that $X_i$ is an immediate cause of $X_j$ relative to $\VVV$, or in short, that $X_i$ is a \textbf{parent} of $X_j$. If a variable $X_i$ is a parent of (a parent of a parent of ...) a variable $X_j$, say that $X_j$ is a \textbf{descendant} of $X_i$. For convenience, we count every variable as its own descendant. We will refer to $G$ simply as a \textbf{(causal) graph}, dropping `directed acyclic', because only directed acyclic graphs are considered in this paper. If a graph $G$ and a joint distribution $P$ are so connected that each variable in $G$ is $P$-independent of its non-descendants given all of its parents (with respect to graph $G$), say that graph $G$ and distribution $P$ satisfy the \textbf{Markov condition}, that $G$ is \textbf{Markov} to $P$, and that $(G, P)$ is a \textbf{causal Bayesian network} (CBN). The Markov condition, as the defining condition of CBNs, is often taken for granted as a necessary connection between causal graphs and joint distributions---between causation and probability.

The Markov condition can be conveniently expressed in another way. Let $\VVV_1, \VVV_2$, and $\VVV_3$ be disjoint subsets of the given, fixed set $\VVV$ of variables. Understand $\VVV_1 \indep \VVV_2 \given \VVV_3$ as the statement saying that $\VVV_1$ and $\VVV_2$ are independent given $\VVV_3$. A graph $G$ is said to \textbf{entail} a conditional independence statement $\VVV_1 \indep \VVV_2 \given \VVV_3$ if that statement holds with respect to every joint distribution to which $G$ is Markov. Let $\II(G)$ denote the set of the conditional independence statements that $G$ entails. Let $\II(P)$ denote the set of the conditional independence statements that hold with respect to $P$. Then it is well-known that $G$ and $P$ satisfy the Markov condition if and only if
$$\II(G) \,\subseteq\, \II(P) \,.$$ 
Now, consider the following stronger condition: 
$$\II(G) \,=\, \II(P) \,.$$
This condition says that the conditional independence statements entailed by $G$ are exactly those that hold with respect to $P$; in that case, say that $G$ is \textbf{faithful} to $P$, and that the CBN $(G, P)$ satisfies the \textbf{faithfulness condition}. With respect to an \textbf{unfaithful} CBN $(G, P)$, at least one conditional independence statement $\sigma$ turns out to hold (i.e., $\sigma \in \II(P)$) even though it is not required to hold by the Markov condition (i.e., $\sigma \notin \II(G)$). 

	\begin{figure}[ht]
	\centering	\includegraphics[width=0.95\textwidth]{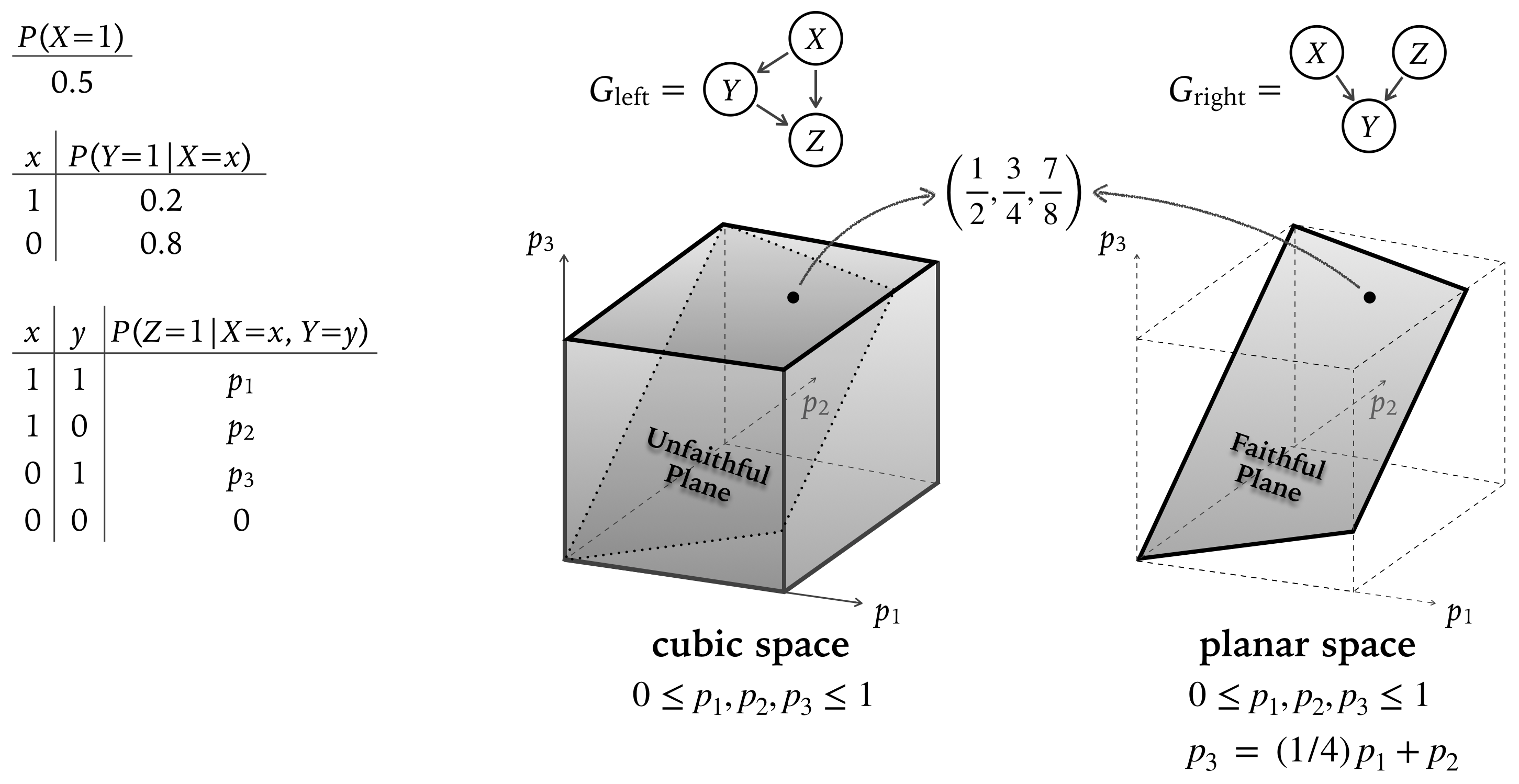}
	\caption{Two spaces of causal Bayesian networks}
	\label{fig-two-spaces}
	\end{figure}

To illustrate, let $\VVV$ contain only three binary variables $X$, $Y$, and $Z$. Consider the two causal graphs depicted in figure \ref{fig-two-spaces}, $G_{\mathrm{left}}$ and $G_{\mathrm{right}}$. Instead of thinking about all joint distributions on $\VVV$, for the sake of visualization let's consider just the joint distributions that are defined by the three tables on the left of the same figure, with three parameters $p_1, p_2$, and $p_3$ taking values in the unit interval. So those parameterized distributions form a unit cube. The design of this parameterized family ensures that the left causal graph $G_\lee$ is Markov to each of those distributions. So the points in the left cube in the same figure represent all the CBNs $(G_\lee, P)$ with $P$ in the parameterized family.

The right graph $G_\rii$, on the other hand, turns out to be Markov {\em only} to the distributions $(p_1, p_2, p_3)$ under the constraint $p_3 = (1/4) p_1 + p_2$, which defines the trapezoidal plane on the right in figure \ref{fig-two-spaces}. Here is why. The left graph $G_\lee$ entails only the conditional independence statements that are trivially true (i.e., true simply in virtue of the probability calculus). But the right graph $G_\rii$ entails one more conditional independence statement: $\{X\} \indep \{Z\} \given \varnothing$, which says that $X$ and $Z$ are independent. To satisfy this additional independence relation is provably to satisfy the equation $p_3 = (1/4) p_1 + p_2$ that defines the trapezoidal plane on the right. 

The cubic space of CBNs on the left embeds a copy of the trapezoidal plane, as depicted in the same figure. So every CBN $(G_\lee, P)$ on the left, embedded trapezoid shares the same joint distribution $P$ with its counterpart CBN $(G_\rii, P)$ on the right. Due to this sharing of the same joint distribution, every CBN $(G_\lee, P)$ on the left, embedded trapezoid is unfaithful because it satisfies the independence between $X$ and $Z$, which goes beyond the conditional independence statements entailed by the left causal graph $G_\lee$. So the left, embedded trapezoid contains only unfaithful CBNs; accordingly, call it the {\bf unfaithful plane}. The sharing of the same joint distribution is an important source of the difficulty of causal learning, as we will see below.

A {\bf causal learning problem} is represented by a triple $(\VVV, \SSS, \HH)$ whose three components are understood as follows: 
	\ope  
	\im {\sc (Variable)} A certain causal system is under study and assumed or known to be accurately represented by a CBN on a given set $\VVV$ of variables. 
	\im {\sc (State)} The true, unknown CBN is assumed, and only assumed, to be in a given set $\SSS$ of CBNs over those variables. Each element of $\SSS$ is a CBN $(G, P)$ understood as a possible state of the causal world, called a {\bf causal state}, in which the true joint probability distribution is $P$ and the true causal structure is $G$. So $\SSS$ is the space of the possible causal states under consideration. 
	\im {\sc (Hypothesis)} The goal is to learn, from non-experimental data, the truth among certain competing hypotheses that form a given set $\HH$---hypotheses about the causal structure of the true, unknown CBNs. In every causal state in $\SSS$, exactly one causal hypothesis in $\HH$ is true. 
	\ede 

Given a causal learning problem $(\VVV, \SSS, \HH)$, a learning method for tackling that problem is, roughly, a function that maps each possible ``data set'' to a hypothesis in $\HH$; such a learning method might perform ``well'' in one causal state but ``poorly'' in another, and will be evaluated in terms of its (varying) performances in all the causal states in the given state space $\SSS$. Those rough ideas can be made precise as follows.

To define data sets, assume for simplicity that we can observe the value of every variable in $\VVV$. Fix an enumeration of the variables $X_1, X_2, \ldots, X_K$ in $\VVV$. Let $\XX = (X_1, X_2, \ldots, X_K)^\mathrm{T}$ be the column vector of those variables, in the order of the given enumeration. The observation of the $i$-th instance of the causal system under study will be represented by a random vector $\XX_i = (X_{1,i}, X_{2,i}, \ldots, X_{K,i})^\mathrm{T}$, where $X_{k,i}$ represents the observation of the $k$-th variable $X_k$ in the $i$-th instance of the causal system. Assume that observations $\XX_1, \XX_2, \ldots, \XX_i, \ldots$ of different instances are independent and identically distributed (IID). So, if $(G, P)$ is the true CBN, then $\XX_i \sim \XX \sim P$ for each $i \ge 1$. A \textbf{data set} of sample size $n$ (over the set $\VVV$ of variables), written $(\xx_1, \ldots, \xx_n)$, is a realization of the $n$ observations $(\XX_1, \ldots, \XX_n)$. The collection of all such data sets, written $\Data(\VVV)$, contains the possible inputs considered in this paper.

A \textbf{learning method} for tackling a causal learning problem $(\VVV, \SSS, \HH)$ is formally a function from $\Data(\VVV)$ to $\HH$---a function $\hat{H}$ that maps any data set $(\xx_1, \ldots, \xx_n)$ of any sample size $n$ (over the set $\VVV$ of variables) to a hypothesis in $\HH$, denoted by $\hat{H}(\xx_1, \ldots, \xx_n)$. To clarify, $\hat{H}(\xx_1, \ldots, \xx_n)$ is a {\em specific} causal hypothesis, one that the learning method $\hat{H}$ outputs/accepts when it receives a concrete data set $(\xx_1, \ldots, \xx_n)$. In contrast, $\hat{H}(\XX_1, \ldots, \XX_n)$ is a {\em random} causal hypothesis---a random variable denoting the causal hypothesis that $\hat{H}$ outputs given a random observation $(\XX_1, \ldots, \XX_n)$ of sample size $n$. While a learning method is construed abstractly as a mere input-output relation, a \textbf{learning algorithm} is understood as a concrete instruction that (efficiently or inefficiently) implements some learning method. This paper focuses on the abstract level of learning methods rather than the concrete level of learning algorithms.

In a causal state $s = (G, P)$, a learning method $\hat{H}$ might perform well or poorly given a sample size $n$, and the performance is captured by the following probabilities:
	\begin{eqnarray*}
	\textbf{success probability} &=& \mathbb{P}_s \left( \hat{H}(\XX_1, \ldots, \XX_n) = H_s\right) ,
	\\
	\textbf{error probability} &=& \mathbb{P}_s \left( \hat{H}(\XX_1, \ldots, \XX_n) \not= H_s\right) ,
	\end{eqnarray*}
where $\Prob_s$ denotes the sampling distribution true in causal state $s = (G, P)$, namely the $\infty$-fold probability measure generated by $P$ under the IID assumption, and $H_s$ denotes the causal hypothesis in $\HH$ that is true in causal state $s = (G, P)$.

\section{Key Definitions: Modes of Convergence}\label{sec:setup2}

The crux of the matter can be understood this way: When the learning problem $\PPP$ in question makes too weak a background assumption---that is, when the state space $\SSS$ under consideration is too inclusive---then (almost) all familiar desiderata for learning methods become too high an ideal to be achievable. For example, consider the familiar desideratum of {\em statistical consistency}, which can be defined in the present setting as follows. A learning method $\hat{H}$ for tackling a causal learning problem $\PPP$ is said to \textbf{converge (in probability) to the truth} in a causal state $s = (G, P)$ if, in state $s$, the success probability of $\hat{H}$ approaches $1$ as the sample size $n$ increases indefinitely, or in symbol:
	$$\mathbb{P}_s \left( \hat{H}(\XX_1, \ldots, \XX_n) = H_s\right) \to 1  \quad\mathrm{as}\quad n \to \infty \,.$$
If learning method $\hat{H}$ converges to the truth in every state in the given state space $\SSS$, say that it is {\bf statistically consistent}, or more intuitively, say that it converges to the truth {\bf everywhere} with respect to the given learning problem $\PPP$. Statistical consistency alone may not be enough for making a good learning method, but it is usually taken as one of the {\em minimal qualifications} for making a good learning method if this qualification can possibly be met. Indeed, there is a familiar, higher evaluation standard: A learning method $\hat{H}$ for tackling a causal learning problem $\PPP$ is said to converge to the truth with {\bf global uniformity} (aka {\bf uniform consistency}) if  
	$$
	\inf_{s \in \SSS} \mathbb{P}_{s} \! \left( \hat{H}(\XX_1, \ldots, \XX_n) = H_{s}\right) \to 1  \quad\mathrm{as}\quad n \to \infty \,.
	$$

When the state space $\SSS$ is too inclusive, everywhere convergence can be easily unachievable, let alone the stronger condition of globally uniform convergence. This can already be seen from the example illustrated in figure \ref{fig-two-spaces}. To be more specific, consider the joint distribution $P^*$ parameterized by $\left(p_1, p_2, p_3\right) = \left(\frac{1}{2}, \frac{3}{4}, \frac{7}{8}\right)$, as indicated in that figure. Suppose that the state space $\SSS$ contains at least the unfaithful causal state $s = (G_\lee, P^*)$ on the left and its counterpart $s' = (G_\rii, P^*)$ on the right; they share the same joint distribution $P^*$. Same joint distribution, same sampling distribution; so $\mathbb{P}_{s} = \mathbb{P}_{s'}$. Now, suppose that some hypothesis in $\HH$ is true in one of those two states but false in the other. Then any learning method for tackling the present problem $\PPP$ has to fail to converge to the truth in one of those two causal states. That is, the problem $\PPP$ just described is too hard to allow the possibility of everywhere convergence, let alone the stronger mode of convergence, globally uniform convergence.

When high standards are unachievable, it is natural to look for lower standards and see whether they are achievable. So, define some weaker modes of convergence as follows. 

The domain of convergence should be extended as far as possible. Accordingly, a learning method $\hat{H}$ for tackling a problem $\PPP$ is said to converge  (in probability) to the truth \textbf{on a maximal domain} if no other learning method converges (in probability) to the truth in at least all causal states in $\SSS$ where $\hat{H}$ does and in strictly more causal states in $\SSS$. 

The domain of convergence should, if possible, be extended to cover at least ``almost everywhere'', which will be defined quite standardly as in geometry and topology: ``almost everywhere'' as ``everywhere'' except on a ``nowhere dense'' subspace. Consider a very standard metric for measuring the distance between probability measures, the \textbf{total variation distance}, which is defined by: $\Delta(P, P') = \sup_{A} \left| P(A) - P'(A) \right|$, for any probability measures $P$ and $P'$. Choose a metric $\delta$ defined on the set of all causal graphs over $\VVV$ (any metric would do), hold $\delta$ fixed, and let $\delta(G, G')$ measure the distance between two causal graphs $G$ and $G'$. The distance between two causal states $s = (G, P)$ and $s' = (G', P')$ will be measured by a certain metric $d$, defined by $d(s, s') = \delta(G, G') + \Delta(P, P')$, i.e., the sum of the distance between the two causal structures and the distance between the two joint distributions. An {\bf open ball} centered at a causal state $s$ is a set taking this form:
	\begin{eqnarray*}
	B_\epsilon(s)
		&=& \left\{ s' \in \SSS: d(s, s') < \epsilon \right\} \,,
	\end{eqnarray*}
where the radius $\epsilon$ is required to be positive. This turns the state space $\SSS$ into a topological space, whose {\bf open sets} are defined as unions of open balls. The specific distance functions used to define the open sets are inessential to this paper; it is the open sets that are essential.\footnote{We are indebted to Kevin T. Kelly for suggesting to us this topology, which significantly improves on the topology used in an earlier draft of this paper.} With respect to a topological space held fixed, a subset $X$ is said to be \textbf{(topologically) negligible}, aka \textbf{nowhere dense}, if for every open set/ball $B$, there is some open set/ball $B'$ that is nested within $B$ and disjoint from $X$. In that case, $X$ has an open ``hole'' $B'$ in every local neighborhood $B$ in the topological space; it is like a slice of Swiss cheese incredibly full of open holes. A learning method $\hat{H}$ for tackling a problem $\PPP$ is said to converge (in probability) to the truth \textbf{almost everywhere} if $\hat{H}$ converges  (in probability) to the truth in all causal states in $\SSS$ except on a nowhere dense subset of $\SSS$. 

Now we turn to uniformity of convergence. A learning method $\hat{H}$ for tackling a problem $\PPP$ is said to converge (in probability) to the truth with \textbf{adherent local uniformity} if, for any causal state $s \in \SSS$, if $\hat{H}$ converges (in probability)  to the truth in $s$, then $\hat{H}$ converges (in probability) to the truth uniformly on some open neighborhood $B_\epsilon(s)$ of $s$ in the state space $\SSS$, or in symbol, there exists a radius $\epsilon > 0$ such that
	$$
	\inf_{s' \in B_\epsilon(s)} \mathbb{P}_{s'} \! \left( \hat{H}(\XX_1, \ldots, \XX_n) = H_{s'}\right) \to 1  \quad\mathrm{as}\quad n \to \infty \,.
	$$
This means that, in every causal state in which the learning method converges to the truth, the error probability can be made not just low but {\em stably} low: the performance of such a learning method remains good {\em even when} the true, unknown CBN is vulnerable to a sufficiently small perturbation over which we do not have control. As we will see below from lemma \ref{lem:smallball}, when a perturbation is sufficiently small, it will be a perturbation that changes the joint distribution only slightly but without a change in the causal structure---that is, with {\em adherence} to the unchanged causal structure.


\section{Main Result \& Discussion} \label{sec:result1}

Our main result addresses the question of where the domain of convergence should be extended if it cannot be extended to cover everywhere. Accordingly, when a learning method converges to the truth in state $s$, say that it has the convergence property be {\bf secured} in $s$; otherwise say that it has the convergence property be {\bf sacrificed} in $s$. Following \cite{spirtes2000causation}, we will focus on the task of learning a specific kind of causal hypothesis, called Markov equivalence hypothesis, which can be defined as follows. Two graphs $G$ and $G'$ are said to be \textbf{Markov equivalent} if $\II(G) = \II(G')$---that is, if $G$ and $G'$ are graphically so similar that they entail exactly the same conditional independence statements. For example, the two graphs $G_\lee$ and $G_\rii$ depicted in figure \ref{fig-two-spaces} are {\em not} Markov equivalent, for the right one entails the independence between $X$ and $Z$, which is not entailed by the left one. Each graph $G$ generates a \textbf{Markov equivalence class} $[G]$, defined as the set of all the graphs that are Markov equivalent to $G$. Each such class $[G]$ generates a {\bf Markov equivalence hypothesis} $H_G$: ``The causal graph of the true CBN is in Markov equivalence class $[G]$.'' The rationale for focusing on this kind of hypothesis will be explained below (in the discussion that follows the statement of the main result).

\begin{theorem} 
\label{the:general1}
Let $\PPP$ be any causal learning problem such that $\VVV$ is a finite set of categorical variables, $\SSS$ is the state space consisting of all causal states on $\VVV$ (i.e., all causal Bayesian networks on $\VVV$), and $\HH$ is the hypothesis set consisting of all the Markov equivalence hypotheses about $\VVV$. Suppose that there are at least two variables in $\VVV$. Then we have:
	\op
	\im[1.] Learning problem $\PPP$ is so hard that it admits of no learning method that achieves the standard of everywhere convergence to the truth  (as Spirtes et al. 2000 have already shown), let alone the higher standard of global uniform convergence.
		
	\im[2.] But learning problem $\PPP$ is at least easy enough to admit of a learning method that achieves this lower standard: convergence to the truth $(a)$ almost everywhere, $(b)$ on a maximal domain, and $(c)$ with adherent local uniformity. 

	\im[3.] For any learning method tackling problem $\PPP$, if it achieves at least that much, namely the joint mode of convergence $(a)$+$(b)$+$(c)$, then it has the convergence property be
		\op
		\im[3.1]  secured in (at least) every faithful causal state in $\SSS$, 
		\im[3.2] sacrificed in (at least) every unfaithful causal state in $\SSS$ that shares its joint probability distribution with some faithful causal state. 
		\ed 
    \ed
\end{theorem}
The impossibility result in clause 1 is familiar, as mentioned above.\footnote{This is the only clause whose truth depends on the assumption that there are at least two variables in $\VVV$.} The possibility result in clause 2 is proved by a (long) sequence of constructions and verifications, detailed in appendix \ref{sec:proof-of-existence}. Clause 3 follows immediately from theorem \ref{the:general2}, to be presented in the next section and proved below (in section \ref{sec:proof-crucial-lemma} and appendix \ref{sec:proof-of-theorem}). It is also possible to draw some pictures to illustrate why clause 3 holds in certain special cases; see appendix \ref{sec:clause3} for details. 

The three clauses of this theorem are best understood with the help of figure \ref{fig-modes}: every joint mode of convergence outside the shaded area is unachievable, thanks to clause 1; every one inside the shaded area is achievable, thanks to clause 2. So the first two clauses determine the highest achievable mode of convergence considered in this paper---the one marked with a star in figure \ref{fig-modes}. Then, clause 3 says what it takes to achieve the highest achievable one. The rest of this section takes a closer look at the three clauses of this theorem in turn. 

Clause 1 reports the impossibility of securing the convergence property everywhere, so sacrifices have to be made {\em somewhere}. The question is: {\em Where should sacrifices be made?} The crux of the matter is that the learning problem in question is very hard, so much so that (almost) all familiar standards for evaluating learning methods are too high an ideal to be achievable. In that case, there appears to be no achievable evaluation standard on the table, and hence no constraint on the candidate pool of good learning methods---in the present case, there appears to be no constraint on where the convergence property should be sacrificed. This is the source of the problem. In response, this paper pursues a general strategy: When confronted with a learning problem that is too hard to make it possible to achieve any of the familiar, higher evaluation standard (clause 1 of theorem \ref{the:general1}), we should first proceed by looking for a lower evaluation standard that is desirable and achievable (clause 2 of theorem \ref{the:general1}); if we manage to find one, we should then determine what it takes for a learning method to achieve that lower standard (clause 3 of theorem \ref{the:general1}). 

So, underlying clause 2 is the task of seeking a lower, achievable standard for the evaluation of causal learning methods. When everywhere convergence is impossible, it is still desirable to extend the domain of convergence as far as possible, preferably missing only a region that is negligible in a mathematically rigorous sense. This desideratum is made precise in terms of the modes of convergence $(a)$ and $(b)$ mentioned in clause 2. When globally uniform convergence is impossible---that is, when a high success probability cannot be obtained and retained under any perturbation, it is still desirable, if possible, to retain it under any perturbation of a limited kind. This desideratum is made precise in terms of the mode of convergence $(c)$ mentioned in clause 2. It turns out that those weaker modes of convergence, $(a)$, $(b)$, and $(c)$, are not just each achievable, but jointly achievable, as clause 2 shows.

To clarify, when it is possible to simultaneously achieve the three proposed modes of convergence, this achievement is only necessary, rather than sufficient, for making a good learning method. So, if there is any additional desideratum that can be achieved jointly with those three, it should be added to the stock of the evaluation standards in use---in order to further constrain the candidate pool for good learning methods. The point is that, as clause 3 shows, those three modes of convergence already work together to impose an interesting, significant constraint: to achieve {\em at least} those three simultaneously, a learning method based on non-experimental data has to secure the convergence property in at least every faithful causal state---this is mandatory rather than optional. So clause 3 justifies the standard design practice of sacrificing the convergence property {\em only} in unfaithful causal states. 


The above illustrates how clauses 2 and 3 work together to address the question left by clause 1, the question of where the convergence property should be secured or sacrificed. The more traditional reaction to the impossibility result in clause 1 is to make the assumption that the true, unknown CBN is faithful. This restricts the state space $\SSS$ to the set $\SSS_\mathrm{faith}$ of all faithful causal states, which, in a sense, restores the possibility of statistical consistency: everywhere convergence is achievable with respect to the modified, easier learning problem $(\VVV, \SSS_\mathrm{faith}, \HH)$, as \cite{spirtes2000causation} show. But to simply assume that the true, unknown CBN is faithful is to take for granted a specific answer to the question of where sacrifices should be made. The proposal of this paper is that causal learning theory need not be developed on the assumption that the true, unknown CBN is faithful, or on any other variant of the faithfulness assumption. Instead, the problem posed by clause 1 can, and should, be addressed by the general strategy that underlies clauses 2 and 3: look for what can be achieved, and achieve the best we can have. So, although the above theorem is stated in an unusual way, with the first clause being nothing but a familiar result, it is so stated in order to emphasize that the difficulty posed by clause 1 should be addressed by something like clauses 2 and 3. 

The above theorem is limited in some ways. First, it only concerns a specific kind of hypothesis space: the set of the {\em Markov equivalence} hypotheses about the given set $\VVV$ of variables. This choice of a hypothesis space is made in this paper for a reason. A causal learning problem can certainly have a hypothesis space of some other kind. For example, \cite*{shimizu2006linear} study the causal learning problem $(\VVV', \SSS', \HH')$ such that $\VVV'$ is a set of continuous variables and the state space $\SSS'$ is restricted to the (so-called) linear non-Gaussian acyclic causal models (or CBNs in which each variable is required to be a linear function of its parents plus an random noise term whose distribution must be non-Gaussian, but see \cite*{zhang2009identifiability} for the extent to which this restriction might be relaxed). Under such parametric assumptions, it is unnecessary to restrict attention to Markov equivalence hypotheses or care much about the faithfulness condition.  
By way of contrast, it is the need to learn at least the true Markov equivalence hypothesis without making strong parametric assumptions that motivates \cite{spirtes2000causation} to make the faithfulness assumption---to assume away unfaithful CBNs from the state space under consideration. So, to explore the possibility of developing causal learning theory without assuming away any unfaithful CBNs, a good starting point is to study the task of learning the true Markov equivalence hypothesis, as pursued in this paper.  

The above theorem is also limited in another way: clause 3 only says that the convergence property has to be secured in {\em at least} all the faithful causal states. This raises some questions: {\em Exactly} where does the convergence property have to be secured? Also, {\em exactly} where does it have to be sacrificed? And {\em exactly} where is the sacrifice only optional but not mandatory? These questions are answered by a strengthening of the above theorem, to be presented in the next section.


\section{Main Result Strengthened} \label{sec:result2}

To strengthen clause 3 of the preceding theorem, some more definitions are required. 

A condition weaker than faithfulness is called (Pearl's) minimality \citep{pearl2009causality}. 
Say that $G$ is \textbf{minimal} to $P$ if there exists no graph $G'$ such that $\II(G) \subset \II(G') \subseteq \II(P)$. Call a causal state $(G, P)$ {\bf minimal} if $G$ is minimal to $P$. The term `minimal' can be understood intuitively this way. Suppose that $P$ is the true joint probability distribution. So the conditional independence {\em facts} are those in $\II(P)$. Let's try to explain (some of) those facts by postulating a causal structure $G$. Assuming the causal Markov condition, we have to postulate a causal graph $G$ with $\II(G) \subseteq \II(P)$. So, of the conditional independence facts in $\II(P)$, those included in $\II(G)$ are explained (namely, entailed) by the postulated causal structure $G$ but those in $\II(P) \smallsetminus \II(G)$ are left unexplained (yet)---at least they cannot be explained by (lack of) causation if we postulate $G$. If we postulate a causal structure $G$ that is minimal to $P$, it means that only a {\em minimal} set of conditional independence facts is left unexplained by (lack of) causation. In the limiting case that $\II(G) = \II(P)$, no conditional independence fact is left unexplained.
    

Call a causal state $(G, P)$ \textbf{u-minimal (unambiguously minimal)} if $G$ is minimal to $P$ and every graph minimal to $P$ is Markov equivalent to $G$. As indicated by the Venn diagram in figure \ref{fig-venn-diagram}, faithfulness is strictly stronger than u-minimality, which is in turn strictly stronger than minimality \citep{zhang2013comparison}.\footnote{For an example of a u-minimal but not faithful causal state, see \citeauthor{zhang2013comparison} (\citeyear{zhang2013comparison}, pp. 433--434). For examples of minimal but not u-minimal causal states, see \citeauthor{zhang2013comparison} (\citeyear{zhang2013comparison}, p. 431).}

	\begin{figure}[ht]
	\centering	\includegraphics[width=0.8\textwidth]{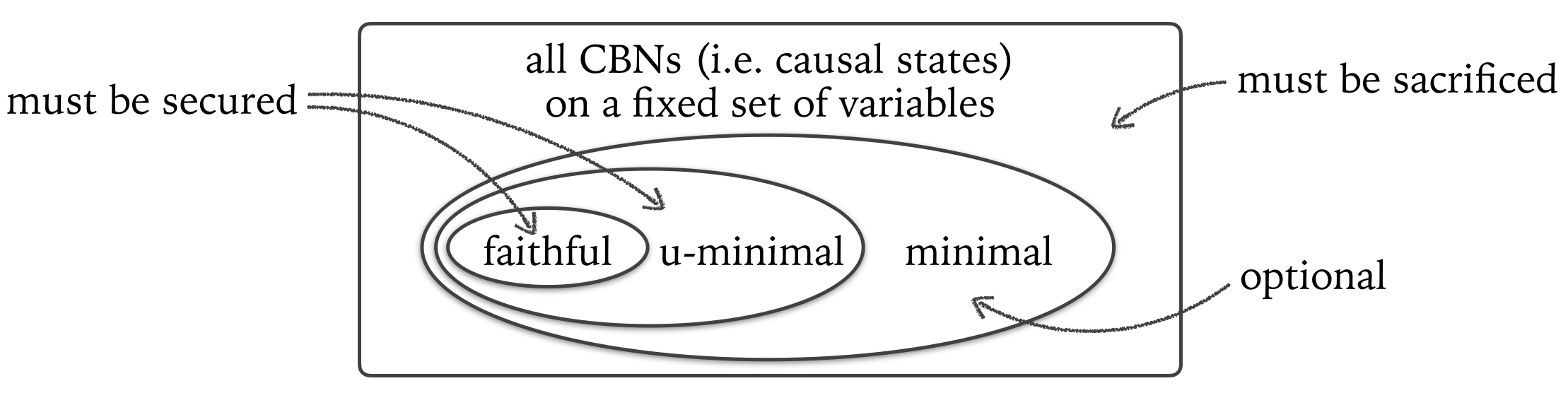}
	\caption{Summary of theorem \ref{the:general2}}
	\label{fig-venn-diagram}
	\end{figure}

\begin{theorem}\label{the:general2}
Continuing from theorem 1, we have: to achieve at least the joint mode of convergence $(a)$+$(b)$+$(c)$, the convergence property must be secured in all u-minimal causal states, must be sacrificed in all non-minimal causal states, and is optional for the other causal states (as summarized in figure \ref{fig-venn-diagram}). To be more precise:
	\ope
	\im For any learning method tackling problem $\PPP$, if it achieves the joint mode of convergence $(a)$+$(b)$+$(c)$, then it has the convergence property be
	\op
	\im[1.1] secured in all u-minimal causal states (including all faithful causal states) in $\SSS$, 
	\im[1.2] sacrificed in all causal states in $\SSS$ that are not minimal. 
	\ed 

	\im For any other causal state $s$ in $\SSS$ (i.e. minimal but not u-minimal), some but not all learning methods achieving $(a)$+$(b)$+$(c)$ converge to the truth in $s$.   
    \ede
\end{theorem}
The proof is in appendix \ref{sec:proof-of-theorem}. 

But here is the strategy that underlies the proof. It is first shown that we only need to use the two modes of convergence $(a)$ ``almost everywhere'' and $(c)$ ``adherently locally uniform'' to force the convergence property to be sacrificed in all causal states that are not minimal. This means that $(a)$+$(c)$ alone suffices to prove clause 1.2. This also means that, if the convergence property must be secured somewhere, $(a)$+$(c)$ already requires that it be secured {\em only within} the range of the minimal causal states. Then, within that particular range, the domain of convergence is extended as much as possible in order to achieve the mode of convergence $(b)$ ``maximal domain''. It is shown that, whenever the domain of convergence has not been extended enough to cover the set of all u-minimal causal states, it can always be extended {\em further} to do so. This means that, within the range of the minimal causal states, any maximal domain of convergence must cover at least the set of all u-minimal causal states, which leads to a proof of clause 1.1. To recap: modes $(a)$ and $(c)$ are used to prove clause 1.2, which is then used together with mode $(b)$ to prove clause 1.1. Clause 2 makes a pair of existence claims for each causal state $s$ that is minimal but not u-minimal: some of the learning methods that achieve $(a)$+$(b)$+$(c)$ converge to the truth in $s$, but some others do not. Those two existence claims are proved with the help of the techniques developed for proving the existence claim (clause 2) of theorem \ref{the:general1}. See appendix \ref{sec:proof-of-theorem} for the complete proof.

\section{Proof of an Important Lemma} \label{sec:proof-crucial-lemma}

This section states and proves what we call {\em the sacrifice lemma} (lemma \ref{lem-crucial} below), whose proof is particularly revealing because it {\em explains} why sacrificing the convergence property in certain causal states is a necessary cost of something good. This will lead to a proof of clause 1.2 of theorem \ref{the:general2}.

\begin{lemma}\label{lem:smallball}
Suppose that $\VVV$ is a finite set of variables, and that $\SSS$ is the set of all causal states (i.e., CBNs) on $\VVV$. With respect to the topological structure defined above (in section \ref{sec:setup2}), we have: there exists a (sufficiently small) radius $\epsilon^* >0$ such that, for any causal state $s = (G, P) \in \SSS$, the open ball $B_{\epsilon^*}(s)$ centered at $s$ with radius $\epsilon^*$ contains only causal states that share the same causal graph, namely $G$. 
\end{lemma}

\begin{proof}
Let $\delta$ be any metric chosen (in section \ref{sec:setup2}) to measure the distances between causal graphs, and let $\delta_{\min}$ be the minimal distance measured by $\delta$ between two distinct causal graphs on $\VVV$. We have that $\delta_{\min} > 0$, for two reasons: first, $\delta$ as a metric must assign a nonzero distance to any pair of distinct causal graphs; second, there are only finitely many causal graphs on the finite set $\VVV$. Let the sought radius $\epsilon^*$ be $\delta_{\min}$. Consider any two causal states $s = (G, P)$ and $s' = (G', P')$ on $\VVV$ that are less $\epsilon^*$-away from each other; that is, $d(s, s') < \epsilon^*$. It suffices to show that $G = G'$. Note that $\delta(G, G') + \Delta(P, P') = d(s, s') < \epsilon^* = \delta_{\min}$. Hence $\delta(G, G')$ is less than $\delta_{\min}$, the minimal distance between two distinct causal graphs on $\VVV$. So $G = G'$, as desired.
\end{proof}

With respect to a causal learning problem $\PPP$, the {\bf domain of convergence} of a learning method is the set of the causal states in $\SSS$ in which that learning method converges (in probability) to the truth. Then we have:

\begin{lemma}\label{lem:dense-open}
	If a learning method tackling a causal learning problem $\PPP$ converges to the truth almost everywhere, then its domain of convergence is dense in $\SSS$. If it converges to the truth with adherent local uniformity, then its domain of convergence is open in $\SSS$.
\end{lemma}

\begin{proof}
	Immediate from definitions. 
\end{proof}

Having a dense and open domain of convergence implies a significant constraint on where the convergence property must be sacrificed: 

\begin{lemma}\label{lem-crucial} \!\!{\bf (The Sacrifice Lemma)} \; Let $\PPP$ be a causal learning problem such that $\VVV$ is a finite set of variables (whether or not those variables are categorical, discrete, or continuous), that $\SSS$ is the set of all causal states on $\VVV$, and that $\HH$ is the set of the Markov equivalence hypotheses about $\VVV$. For any learning method $\hat{H}$ tackling problem $\PPP$, if $\hat{H}$ has a dense and open domain of convergence on the state space $\SSS$, then $\hat{H}$ sacrifices the convergence property in every causal state in $\SSS$ that is not minimal. 
\end{lemma}

\begin{proof}
Let $\hat{H}$ be a causal learning method with a dense and open domain of convergence. By lemma \ref{lem:smallball}, there exists a (small) radius $\epsilon^* > 0$ such that, if any two causal states are less than $\epsilon^*$-away from each other, they share the same causal structure. Suppose, for {\em reductio}, that $\hat{H}$ converges to the truth in some non-minimal causal state $s_0 = (G, P) \in \SSS$. Since the domain of convergence is open (by hypothesis), there exists a nonzero radius $\epsilon \le \epsilon^*$ such that $\hat{H}$ converges to the truth in every causal state in the open ball $B_{\epsilon}(s_0)$. Note that the Markov equivalence hypothesis $H_G$ is true in $s_0$, and hence true in every causal state in the open ball $B_{\epsilon}(s_0)$, because $\epsilon \le \epsilon^*$. (See the upper left part of figure \ref{fig-proof-of-lemma} for a picture of the present situation, in which the so-called non-minimal plane represents the set of the non-minimal causal states in which $H_G$ is true). From $s_0$ let's construct causal states $s_1, s_2$, and $s_3$ in the following three steps (as represented by the three arrows in figure \ref{fig-proof-of-lemma}):
	\begin{figure}[ht]
	\centering	\includegraphics[width=0.8\textwidth]{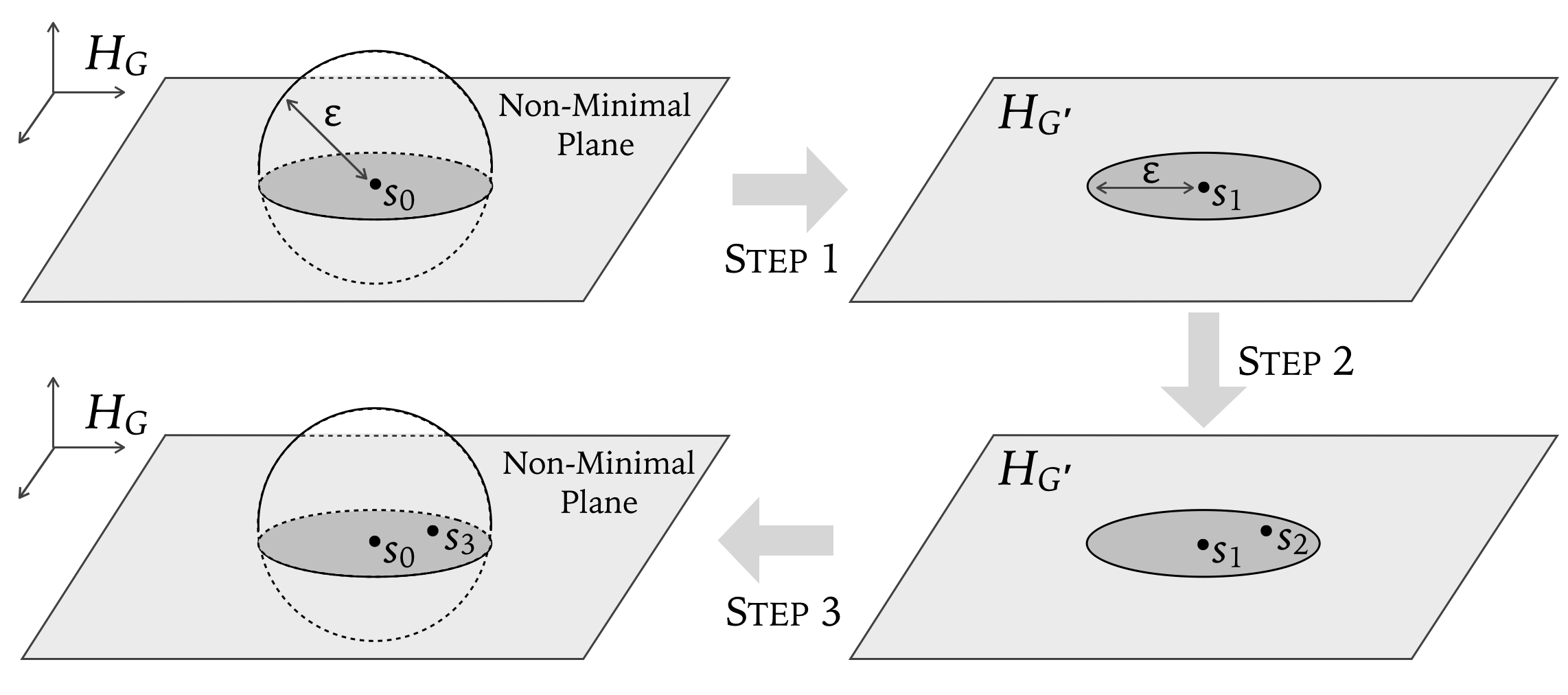}
	\caption{Constructions in the proof of lemma \ref{lem-crucial}}
	\label{fig-proof-of-lemma}
	\end{figure}
	\op 
	\im {\sc Step 1.}  
	Since causal state $s_0 = (G, P)$ is not minimal, there exists a minimal causal state $s_1 = (G', P) \in \SSS$ with $\II(G) \subset \II(G') \subseteq \II(P)$. We have thus constructed $s_1$.
	\\[-1.7em]
	\im {\sc Step 2.}
	Since $\II(G) \subset \II(G')$, $G$ and $G'$ are not Markov equivalent. So $H_{G}$ and $H_{G'}$ are two distinct, incompatible hypotheses. Causal state $s_1 = (G', P)$ has an open neighborhood $B_{\epsilon}(s_1)$ with the same radius $\epsilon$, in which all causal states share the same causal graph $G'$ (because $\epsilon \le \epsilon^*$). So $H_{G'}$ is true in every causal state in that open ball $B_{\epsilon}(s_1)$ (as depicted in the upper right part of figure \ref{fig-proof-of-lemma}). 
		Then, since the domain of convergence is dense (by hypothesis), the open ball $B_{\epsilon}(s_1)$ contains at least one causal state in which $\hat{H}$ converges to the truth $H_{G'}$---now, choose one such causal state $s_2 = (G', P')$. We have thus constructed $s_2$.
	\\[-1.7em]		
	\im {\sc Step 3.}
	Take causal state $s_2 = (G', P')$, replace the graph therein by $G$ to construct an ordered pair $s_3 = (G, P')$. Argue as follows that $s_3$ is indeed a causal state. Note that $G$ is Markov to $P'$ because $\II(G) \subseteq \II(G') \subseteq \II(P')$, where the first subset relation $\II(G) \subseteq \II(G')$ follows from the construction of $G'$ and the second subset relation $\II(G') \subseteq \II(P')$ follows from the fact that $s_2 = (G', P')$ is a causal state (i.e., CBN), which must satisfy the Markov condition. Since $G$ is Markov to $P'$, $s_3 = (G, P')$ is indeed a causal state. We have thus constructed causal state $s_3$.
	\ed 
Causal state $s_3$ has some notable properties. First, $s_3$ is in the open ball $B_{\epsilon}(s_0)$, because $d(s_3, s_0) = \delta(G, G) + \Delta(P', P) = 0 + \Delta(P', P) = \delta(G', G') + \Delta(P', P) = d(s_2, s_1) < \epsilon$. Second, $s_3$ shares with $s_2$ the same joint distribution $P'$ (and hence the same sampling distribution), so $\hat{H}$ converges to the same hypothesis in $s_2$ and in $s_3$, and that particular hypothesis is $H_{G'}$ (by the construction of $s_2$).
	It follows that $\hat{H}$ converges to a falsehood $H_{G'}$ in $s_3 = (G, P')$, which is in $B_{\epsilon}(s_0)$. Therefore, $\hat{H}$ fails to converge to the truth in some causal state in $B_{\epsilon}(s_0)$---contradiction.
\end{proof}

Clause 1.2 of theorem \ref{the:general2} follows immediately from the previous two results: lemmas \ref{lem:dense-open} and \ref{lem-crucial}. Note that the above proof does not restrict the variables in $\VVV$ to be categorical variables. So, clause 1.2 of theorem \ref{the:general2} actually holds for any kinds of variables, be they categorical, discrete, or continuous. 

To summarize, we submit that a causal learning method should, if possible, achieve at least the mode of convergence $(a)$ ``almost everywhere'' plus the mode of convergence $(c)$ ``adherently locally uniform'', which by lemma \ref{lem:dense-open} implies having a dense and open domain of convergence, which by lemma \ref{lem-crucial} incurs a necessary cost: having the convergence property be sacrificed in every non-minimal causal state. 

\section{Closing: Some Possibilities for Future Research}
\label{sec:conclusion}

The main results of this paper are theorems \ref{the:general1} and \ref{the:general2}. Although they concern causal learning problems that involve only categorical variables, we conjecture that they can be generalized to cover some other causal learning problems, such as problems in which all causal states under consideration are linear Gaussian structural equation models. We also think that it should be possible to suitably generalize the main results to cover finite sets of variables of many different kinds. Our optimism is based on two observations. First, the key lemma \ref{lem-crucial} is applicable to any kind of variable. Second, with discrete or continuous variables that have an infinite range of possible values to take, the state space $\SSS$ can be too large to be captured by a finite-dimensional Euclidean space of parameters. In that case, it makes no sense to talk about mathematical negligibility as Lebesgue measure zero, which is popularized in the causal discovery community by \cite{spirtes2000causation}. But it still makes sense to understand mathematically negligible sets in topological terms, as nowhere dense sets or even meager sets (defined as unions of countably many nowhere dense sets). In fact, this is the main reason why we opt for the more applicable, topological conception of negligibility.

\acks{We are indebted to Kevin Kelly, Clark Glymour, Frederick Eberhardt, Christopher Hitchcock, Peter Spirtes, Kun Zhang, Konstantin Genin, and three anonymous referees for their very helpful comments on earlier drafts of this paper. Lin's research was supported by the University of California at Davis Startup Funds. Zhang's research was supported in part by the Research Grants Council of Hong Kong under the General Research Fund LU13600715, and by a Faculty Research Grant from Lingnan University.}

\appendix

\section{Proof of the Existence Result (Clause 2) of Theorem \ref{the:general1}}\label{sec:proof-of-existence}

The existence result (clause 2) of theorem \ref{the:general1} requires a long proof, which is broken down into three parts: We start with some topological preliminaries (appendix \ref{sec:lemma:nowhere-dense}), followed by some statistical preliminaries (appendix \ref{sec:lemma:independence}), before we finally construct a learning method that witnesses the existence claim (appendix \ref{sec:lemma:construction}).

Throughout this appendix, $\VVV$ is assumed to be a finite set of categorical variables, and $\SSS$ is the set of all causal states on $\VVV$.

\subsection{Topological Preliminaries}\label{sec:lemma:nowhere-dense}

Let $k$ denotes the number of assignments of values to all variables in $\VVV$. So, any joint distribution $P$ of $\VVV$ is determined by the (joint) probabilities $p_1, p_2, \ldots, p_k$ that $P$ distributes to those $k$ assignments of values, respectively, and hence $P$ can be identified with a point $(p_1, p_2, \ldots, p_k)$ in the $k$-dimensional Euclidean space $\mathbb{R}^k$. Note that a joint distribution $P$ satisfies a conditional independence statement ``$\UU \indep \VV \given \WW$'' if and only if the following equation holds:
	\begin{eqnarray*}
	P(\UU = \uu, \VV = \vv, \WW = \ww) \cdot P(\WW = \ww) \quad && 
	\\
	- \, P(\UU = \uu , \WW = \ww)  \cdot P(\VV = \vv , \WW = \ww)	 &=& 0\,.
	\end{eqnarray*}
Every term $P(\cdots)$ on the left side is a marginal probability, which can be expressed by a sum of some of the joint probabilities $p_1, p_2, \ldots, p_k$. So the left side can be expressed as a (second-degree) polynomial in $k$ variables $p_1, p_2, \ldots, p_k$. More generally, every conditional independence statement $\sigma =$ ``$\UU \indep \VV \given \WW$'' can be represented by a polynomial function $f_\sigma(x_1, x_2, \ldots, x_k)$ in $k$ real-valued variables in this sense: a joint distribution $P$ satisfies conditional independence statement $\sigma =$ ``$\UU \indep \VV \given \WW$'' if and only if $f_\sigma(p_1, p_2, \ldots, p_k) = 0$.

Hold a causal graph $G$ fixed. Consider an arbitrary joint distribution $P$ that satisfies the Markov condition with $G$ (i.e., can form a CBN with $G$). It is well known that the joint distribution $P$ factors according to the graph $G$ into some conditional, marginal distributions. To be more specific, each joint probability $p_i$ in $P$ can be expressed as the product of some conditional probabilities, each of which is the probability for a variable to take a certain value conditional on its parents (in $G$) taking certain values (see the equations in figure \ref{fig-condi-prob} for an example with three binary variables). 
	\begin{figure}[ht]
	\centering	\includegraphics[width=0.95\textwidth]{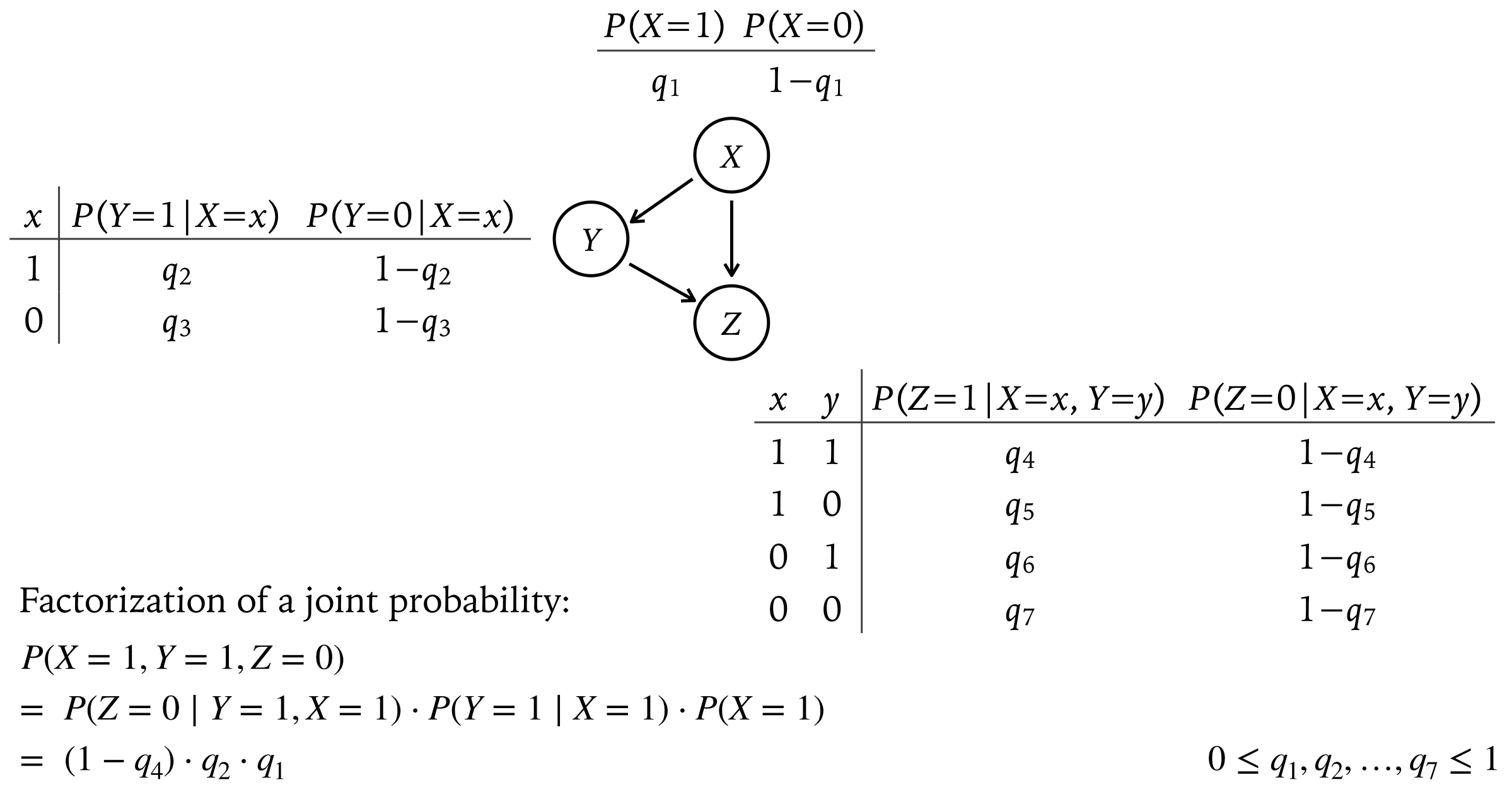}
	\caption{Conditional probability tables}
	\label{fig-condi-prob}
	\end{figure}
For convenience, we adopt the convention that, when a variable has no parent, its probability conditional on ``its parents'' means its unconditional probability. The conditional probabilities just mentioned are the real numbers in the conditional probability tables that are standardly used to represent a Bayesian network (see the tables in figure \ref{fig-condi-prob} for an example). So $P$ is determined by its joint probabilities $p_1, p_2, \ldots, p_k$, each of which can then be expressed as a polynomial in some of the conditional probabilities $q_1, q_2, \ldots, q_m$, where each conditional probability $q_i$ can take any value in the unit interval $[0, 1]$ and $m$ is in general less than $k$ (in figure \ref{fig-condi-prob}, $m = 7 < k = 2^3 = 8$). So a conditional independence statement $\sigma =$ ``$\UU \indep \VV \given \WW$'' can be represented by another polynomial $g_\sigma$, so that 
	\begin{eqnarray*}
	&& \mbox{$P$ satisfies $\sigma =$ ``$\UU \indep \VV \given \WW$''}
	\\
	&\Leftrightarrow & f_\sigma(p_1, p_2, \ldots, p_k) = 0
	\\
	&\Leftrightarrow & g_\sigma(q_1, q_2, \ldots, q_m) = 0
	\end{eqnarray*} 

More generally, let $G$ be a causal graph on a finite set of categorical variables, and let $\DDD_G$ be the set of the joint distributions that are Markov to $G$ (i.e., the joint distributions which can form a CBN with $G$). It is well known that the above provides a smooth parametrization of $\DDD_G$ by an $m$-dimensional parameter space, the $m$-dimensional unit cube $[0, 1]^m$; under this parametrization, every conditional independence statement is represented by a polynomial \citep{meek1995strong}.



\begin{lemma}
	Let $\sigma =$ ``$\UU \indep \VV \given \WW$'' be a conditional independence statement about $\VVV$. The joint distributions in $\DDD_G$ that violate $\sigma =$ ``$\UU \indep \VV \given \WW$'' form an open subset of $\DDD_G$.
\end{lemma}

\begin{proof}
	The joint distributions in $\DDD_G$ that violate $\sigma$ form the set represented by $g_\sigma^{-1}[\mathbb{R} \smallsetminus \{0\}]$, which is an open set because it is a pre-image of an open set ($\mathbb{R} \smallsetminus \{0\}$) under a continuous function (the polynomial function $g_\sigma$).
\end{proof}

\begin{lemma}
	Continuing from the preceding lemma, suppose further that $G$ does not entail the conditional independence statement $\sigma =$ ``$\UU \indep \VV \given \WW$''. Then the joint distributions in $\DDD_G$ that violate $\sigma =$ ``$\UU \indep \VV \given \WW$'' form an open, dense subset of $\DDD_G$.
\end{lemma}

\begin{proof}
By the preceding lemma, the set of the joint distributions in $\DDD_G$ that violate $\sigma =$ ``$\UU \indep \VV \given \WW$'' is open in $\DDD_G$. To show that this set is dense in $\DDD_G$, suppose for {\em reductio} that it is not dense. Then $\DDD_G$ has an open subset on which $\sigma =$ ``$\UU \indep \VV \given \WW$'' is satisfied. Using the conditional probability parametrization described above, it follows that the $m$-dimensional unit cube $[0, 1]^m$ has an open subset $O$ on which $g_\sigma(q_1, q_2, \ldots, q_m) = 0$. It follows that $g_\sigma$ is identically zero on $O$ and all partial derivatives of $g_\sigma$ are also identically zero on $O$, which implies that the Taylor series expansion of $g_\sigma$ only has zero coefficients, which implies that $g_\sigma$ is identically zero on the entire cube $[0, 1]^m$. So the conditional independence statement $\sigma =$ ``$\UU \indep \VV \given \WW$ is satisfied by {\em all} distributions in $\DDD_G$. It follows that $G$ entails $\sigma =$ ``$\UU \indep \VV \given \WW$---contradiction.
\end{proof}

\begin{lemma}
	Let $G$ be a causal graph on $\VVV$, and $\Sigma$ be a finite set of (some, possibly not all) conditional independence statements that $G$ does not entail. Then we have:
	\begin{enumerate}
	\item The joint distributions in $\DDD_G$ that violate every conditional independence statement in $\Sigma$ form an open, dense subset of $\DDD_G$.
	
	\item The joint distributions in $\DDD_G$ that satisfy at least one conditional independence statement in $\Sigma$ form a nowhere dense subset of $\DDD_G$.	
	\end{enumerate}
\end{lemma}

\begin{proof}
	Clause 1 follows immediately from the preceding lemma and the following, familiar fact in general topology: open, dense subsets are closed under finite conjunctions. Clause 2 follows from clause 1 for two reasons: first, the set mentioned in clause 2 is the complement (in $\DDD_G$) of the set mentioned clause 1; second, it is a familiar fact in general topology that any complement of an open, dense, subset is a nowhere dense subset. 
\end{proof}


The above results concern spaces of joint distributions, and can be carried over to spaces of causal states as follows. Let $\SSS_G$ be the topological space of the causal states whose graphs are identical to $G$. A causal state $s = (G, P)$ is said to {\bf satisfy} (or {\bf violate}) a conditional independence statement $\sigma$ if the underlying joint distribution $P$ satisfies (or violates) $\sigma$.

\begin{lemma}
	Let $G$ be a causal graph on $\VVV$, and $\Sigma$ be a finite set of (some, possibly not all) conditional independence statements that $G$ does not entail. Then we have:
	\begin{enumerate}
	\item The causal states in $\SSS_G$ that violate every conditional independence statement in $\Sigma$ form an open, dense subset of $\SSS_G$.
	
	\item The causal states in $\SSS_G$ that satisfy at least one conditional independence statement in $\Sigma$ form a nowhere dense subset of $\SSS_G$.	
	\end{enumerate}
\end{lemma}

\begin{proof}
	Immediate from the previous lemma and the fact that $\DDD_G$ is homeomorphic to $\SSS_G$, with the homeomorphism: $P \mapsto (G, P)$, which maps each joint distribution $P$ in $\DDD_G$ to a causal state $(G, P)$ in $\SSS_G$.  
\end{proof}

\begin{lemma}
	For any causal graph $G$ on $\VVV$, $\SSS_G$ is open in $\SSS$.   
\end{lemma}

\begin{proof}
	Immediate from lemma \ref{lem:smallball} (in section \ref{sec:proof-crucial-lemma}).
\end{proof}

\begin{lemma}\label{lem:open-nowheredense}
	Let $G$ be a causal graph on $\VVV$, and $\Sigma$ be a set of (some, possibly not all) conditional independence statements that $G$ does not entail. Then we have:
	\begin{enumerate}
	\item The causal states in $\SSS_G$ that violate every conditional independence statement in $\Sigma$ form an open subset of $\SSS$.
	
	\item The causal states in $\SSS_G$ that satisfy at least one conditional independence statement in $\Sigma$ form a nowhere dense subset of $\SSS$.	
	\end{enumerate}
\end{lemma}

\begin{proof}
	Immediate from the previous two lemmas, together with the following, familiar facts in general topology: If a set is an open subset of an open subset of a space, it is open in the space. If a set is a nowhere dense subset of a subset of a space, it is nowhere dense in the space.  
\end{proof}

\begin{proposition}\label{pro-open}
	Every causal state $s = (G, P)$ in $\SSS$ has a (sufficiently small) open neighborhood such that, for any causal state $s' = (G', P')$ in that open neighborhood, $G' = G$ and $\II(P') \subseteq \II(P)$. Or in words, every causal state $s$ in $\SSS$ has a (sufficiently small) open neighborhood in which every causal state shares with $s$ the same causal graph and violates at least all the conditional independence statements that $s$ violates.
\end{proposition}

\begin{proof}
	Consider an arbitrary causal state $s = (G, P)$ in $\SSS$. Since $s = (G, P)$ is a causal state, we have that $G$ is Markov to $P$, and it follows that $G$ does not entail any conditional independence statement that $P$ violates. This allows us to apply lemma \ref{lem:open-nowheredense} to graph $G$ together with $\Sigma$ being the set of the conditional independence statements that $P$ violates---namely, those that $s$ violates. Then, by clause 1 of lemma \ref{lem:open-nowheredense}, we have: the set of the causal states in $\SSS_G$ that violate at least all the conditional independence statements that $s$ violates is an open set in $\SSS$. This open set is an open neighborhood of $s = (G, P)$ with the sought properties. The present proposition follows.
\end{proof}

\begin{proposition}\label{pro-unfaithful-nowhere-dense}
	In the space $\SSS$ of all causal states on $\VVV$, the set of the unfaithful ones is nowhere dense and so is the set of the non-minimal ones.
\end{proposition}

\begin{proof}
	Applying the second clause of lemma \ref{lem:open-nowheredense} to any causal graph $G$ on $\VVV$ together with $\Sigma$ being the set of all conditional independence statements that $G$ does not entail, we have: the unfaithful causal states in $\SSS_G$ form a set $\SSS_G^\textrm{unf}$, which is nowhere dense in $\SSS$. Then, the set of the unfaithful causal states in $\SSS$ is nowhere dense in $\SSS$, for two reasons: first, this set is the finite union of the nowhere dense subsets $\SSS_G^\textrm{unf}$ such that $G$ is a causal graph on $\VVV$; second, nowhere dense subsets are closed under finite unions. Moreover, the set of the non-minimal causal states in $\SSS$ is also nowhere dense in $\SSS$, for two reasons: first, it is a subset of a nowhere dense set, namely, the set of the unfaithful ones; second, any subset of a nowhere dense set is nowhere dense.
\end{proof}

\subsection{Statistical Preliminaries}\label{sec:lemma:independence}

\begin{lemma}[Hoeffding's Inequality for Empirical Measures]\label{lem-hoeffding}
Let $\hat{P}_n$ be the empirical distribution (namely, frequency counts) of $n$ observations obtained by IID sampling from a categorical distribution $P$. Then, for any $\epsilon > 0$ and for any sample size $n$, we have:
	\begin{eqnarray*}
	\Prob\left( \Delta(\hat{P}_n, P) < \epsilon\right) &\ge & 1 - 2^k e^{-2n\epsilon^2} \,,
	\end{eqnarray*}
where $\Prob_s$ denotes the sampling distribution generated by $P$, namely the $\infty$-fold probability measure generated by $P$ under the IID assumption, and $\Delta$ is the total variation distance, and $k$ is a constant denoting the number of the categories of $P$.
\end{lemma}

\begin{proof}
It is routine to prove this result in probability theory. Here is one of the standard forms of Hoeffding's Inequality:
	\begin{eqnarray*}
	\Prob\left( \left| \overline{X}_n - \Exp[\overline{X}_n] \right| \ge \epsilon\right) &\le & 2e^{-2n\epsilon^2} \,.
	\end{eqnarray*}
Let ${\cal X}$ be the set of the $k$ given categories. So the set of the relevant events is $2^{\cal X}$. 
	Let $\overline{2^{\cal X}}$ be a subset of $2^{\cal X}$ constructed as follows: for every pair $(A, A')$ of sets that form a partition of ${\cal X}$, choose exactly one of the two sets, $A$ or $A'$, and put it in $\overline{2^{\cal X}}$. Note that the cardinality of $\overline{2^{\cal X}}$ is $2^{k-1}$. 
For each proposition $A \in 2^{\cal X}$, apply Hoeffding's inequality to $| \hat{P}_n(A) - \Exp[\hat{P}_n(A)] |$, which is equal to $| \hat{P}_n(A) - P(A) |$, so we have:
	\begin{eqnarray*}
	\Prob\left( \left| \hat{P}_n(A) - P(A) \right| \ge \epsilon\right) &\le & 2e^{-2n\epsilon^2} \,.
	\end{eqnarray*}
Then we have:
	\begin{eqnarray*}
	\Prob\left( \Delta(\hat{P}_n, P) \ge \epsilon\right) 
	&=& \Prob\left( \max_{A \in 2^{\cal X}}\left| \hat{P}_n(A) - P(A) \right| \ge \epsilon\right)
		\\
	&=& \Prob\left( \bigvee_{A \in 2^{\cal X}} \left(\left| \hat{P}_n(A) - P(A) \right| \ge \epsilon\right)\right) 
		\\ 
	&=& \Prob\left( \bigvee_{A \in \overline{2^{\cal X}}} \left(\left| \hat{P}_n(A) - P(A) \right| \ge \epsilon\right)\right)
		\\
	&\le & \sum_{A \in \overline{2^{\cal X}}} \Prob\left( 
		\left| \hat{P}_n(A) - P(A) \right| \ge \epsilon 
		\right) 
		\\
	&\le & \sum_{A \in \overline{2^{\cal X}}}
	2e^{-2n\epsilon^2} \quad=\quad 2^{k-1} \cdot 2e^{-2n\epsilon^2} \quad=\quad 2^{k} e^{-2n\epsilon^2} \,.
	\end{eqnarray*}
So $\Prob\left( \Delta(\hat{P}_n, P) < \epsilon\right) \ge 1 - 2^{k} e^{-2n\epsilon^2}$, as required.
\end{proof}

\begin{proposition}\label{pro-test-of-independence}
	Every conditional independence statement $\UU \indep \VV \given \WW$ that involves only categorical variables has a test $T$ with the following convergence properties:  
	\ope
	\im on the space of all possible distributions of $\XX = \UU \cup \VV \cup \WW$ that satisfy the independence statement $\UU \indep \VV \given \WW$, test $T$ converges to the truth everywhere with global uniformity;  
	\im on the space of all possible distributions of $\XX = \UU \cup \VV \cup \WW$ that violate the independence statement $\UU \indep \VV \given \WW$, test $T$ converges to the truth everywhere with local uniformity.
   	\ede 
\end{proposition}

\begin{proof}
Let $\UU, \VV, \WW$ be three disjoint sets of categorical variables. We are going to test the hypothesis that $\UU$ and $\VV$ are independent given $\WW$.  Consider an arbitrary joint probability distribution $P$ of $\XX = \UU \cup \VV \cup \WW$. Let $P(\uu, \vv, \ww)$ abbreviate $P(\UU=\uu, \VV=\vv, \WW=\ww)$, and similarly for $P(\uu, \vv)$, $P(\ww)$, $P(\xx)$ etc. Define the following $L_1$-distance of $P$ from the independence of $\UU$ and $\VV$ given $\WW$: 
	\begin{eqnarray*}
	L_1(P) &=& \sum_{\uu, \vv, \ww}
		\big| P(\uu, \vv, \ww) P(\ww) - P(\uu, \ww) P(\vv, \ww) \big| \,,
	\end{eqnarray*}
where $\uu, \vv, \ww$ range over the possible values of $\UU, \VV, \WW$, respectively. Let $\hat{P}_n$ be the empirical distribution (namely, frequency counts) of $n$ observations. (So $\hat{P}_n$ is a random probability distribution of $\XX = \UU \cup \VV \cup \WW$.) It suffices to prove that the existence claim is witnessed by the following test:
	\op
	\im Accept the hypothesis of conditional independence if $L_1(\hat{P}_n) <  \frac{1}{n^{1/4}}$.
	\im Reject that hypothesis otherwise.
	\ed 
We will need to bound $|L_1(P) - L_1(Q)|$, where $P$ and $Q$ are two arbitrary probability distributions of $\XX = \UU \cup \VV \cup \WW$. Bound it as follows:
	\begin{eqnarray*}
	&& |L_1(P) - L_1(Q)| 
		\\
	&=& \Big| \sum_{\uu, \vv, \ww} \big| P(\uu, \vv, \ww) P(\ww) - P(\uu, \ww) P(\vv, \ww) \big|
		- \sum_{\uu, \vv, \ww} \big| Q(\uu, \vv, \ww) Q(\ww) - Q(\uu, \ww) Q(\vv, \ww) \big| \, \Big|
		\\[0.5em]
	&\le & \sum_{\uu, \vv, \ww} \Big| \big| P(\uu, \vv, \ww) P(\ww) - P(\uu, \ww) P(\vv, \ww) \big|
		- \big| Q(\uu,
	 \vv, \ww) Q(\ww) - Q(\uu, \ww) Q(\vv, \ww) \big| \Big|
		\\[0.5em]
	&\le & \sum_{\uu, \vv, \ww} \Big( \big| P(\uu, \vv, \ww) P(\ww) - Q(\uu, \vv, \ww) Q(\ww) \big|
		+ \big| P(\uu, \ww) P(\vv, \ww) - Q(\uu, \ww) Q(\vv, \ww) \big| \Big)
		\\
	&& \;\mathrm{by } \big||a-b| - |a'-b'|\big| \le |a-a'| + |b+b'|
		\\[0.5em]
	&= & \sum_{\uu, \vv, \ww} \big| P(\uu, \vv, \ww) P(\ww) - Q(\uu, \vv, \ww) Q(\ww) \big|
		+ \sum_{\uu, \vv, \ww} \big| P(\uu, \ww) P(\vv, \ww) - Q(\uu, \ww) Q(\vv, \ww) \big| \,.
	\end{eqnarray*}
Then we are going to bound the first and second terms, respectively. Bound the first term as follows:
	\begin{eqnarray*}
	&& \sum_{\uu, \vv, \ww} \big| P(\uu, \vv, \ww) P(\ww) - Q(\uu, \vv, \ww)  Q(\ww) \big|
		\\
	&=& \sum_{\uu, \vv, \ww} \Big| \big(P(\uu, \vv, \ww) - Q(\uu, \vv, \ww)\big) \cdot P(\ww)  + \ Q(\uu, \vv, \ww) \cdot \big(P(\ww) - Q(\ww)\big) \Big| 
		\\
	&\le & \sum_{\uu, \vv, \ww} \Big| \big(P(\uu, \vv, \ww) - Q(\uu, \vv, \ww)\big) \cdot P(\ww) \Big|  + \sum_{\uu, \vv, \ww} \Big| Q(\uu, \vv, \ww) \cdot \big(P(\ww) - Q(\ww)\big) \Big|
		\\
	&\le & \sum_{\uu, \vv, \ww} \big| P(\uu, \vv, \ww) - Q(\uu, \vv, \ww) \big| + \sum_{\ww} \Big( \big| P(\ww) - Q(\ww) \big| \cdot
	\sum_{\uu, \vv} Q(\uu, \vv, \ww) \Big)
		\\
	&\le & \sum_{\uu, \vv, \ww} \big| P(\uu, \vv, \ww) - Q(\uu, \vv, \ww) \big|  + \sum_{\ww} \Big( \big| P(\ww) - Q(\ww) \big| \cdot Q(\ww) \Big)
		\\
	&\le & \sum_{\uu, \vv, \ww} \big| P(\uu, \vv, \ww) - Q(\uu, \vv, \ww) \big| \;+\; \sum_{\ww} \big| P(\ww) - Q(\ww) \big| 
		\\
	&\le & \sum_{\xx} \big|P(\xx) - Q(\xx)\big| \;+\; \sum_{\xx} \big| P(\xx) - Q(\xx) \big|
		\\
	&=& 2 \cdot \sum_{\xx} \big|P(\xx) - Q(\xx)\big| 
		\\
	&=& 4 \, \Delta(P, Q) \,.
	\end{eqnarray*}
The last step follows because $\sum_{\xx} \big|P(\xx) - Q(\xx)\big| = 2 \Delta(P, Q)$, which is a consequence of the fact that $\Delta(P, Q)$ denotes the total variation distance between $P$ and $Q$. The second term can be bounded in the same way:
	\begin{eqnarray*}
	\sum_{\uu, \vv, \ww} \big| P(\uu, \ww) P(\vv, \ww) - Q(\uu, \ww) Q(\vv, \ww) \big| 
	&\le & 4 \, \Delta(P, Q) \,.
	\end{eqnarray*}	
So $|L_1(P) - L_1(Q)|$ can be bounded neatly as follows: 
	\begin{eqnarray}
	|L_1(P) - L_1(Q)| &\le & 8 \, \Delta(P, Q) \,. \label{eqn-bound}
	\end{eqnarray}
Let $P$ be the (unknown) true distribution under the null hypothesis that the conditional independence statement holds. So $L_1(P) = 0$. Let $\hat{P}_n$ be the random empirical distribution generated from $P$ with sample size $n$. Consider the following inequality:
	$$
	L_1(\hat{P}_n) 
	\;=\; \big| L_1(\hat{P}_n) - L_1(P) \big| 
	\;\le\; 8 \, \Delta(\hat{P}_n, P) \;<\;  \frac{1}{n^{1/4}} \,.
	$$
This inequality holds with a probability at least $1 - 2^k e^{-2n\left(\frac{1}{8n^{1/4}}\right)^2 } = 1 - 2^k e^{-\frac{\sqrt{n}}{32}}$ (by inequality (\ref{eqn-bound}) and lemma \ref{lem-hoeffding}), which converges to $1$ as $n$ tends to infinity. Also note that this probability bound holds for all distributions under the null hypothesis. So clause 1 follows. 

Now, let's turn to how the test performs under the alternative hypothesis that the conditional independence statement does not hold. Let $P^*$ be the (unknown) true distribution under the alternative hypothesis. So $L_1(P^*) > 0$. Let $P$ an arbitrary distribution in the open ball $B_{L_1(P^*)/32}\left( P^* \right)$. Let $\hat{P}_n$ be the random empirical distribution generated from $P$ with sample size $n$.
  Consider the following inequality:
	\begin{eqnarray*}
	L_1(\hat{P}_n) 
	&\ge & L_1(P^*) - \big| L_1(P^*) - L_1(P) \big| - \big| L_1(P) - L_1(\hat{P}_n) \big|
		\\
	&\ge & L_1(P^*) - 8\, \Delta(P^*, P) - 8 \, \Delta(P, \hat{P}_n)
		\\
	&>& L_1(P^*) - 8 \left(\frac{L_1(P^*)}{32}\right) - 8 \left(\frac{L_1(P^*)}{16}\right)
		\\
	&=& \frac{1}{4} L_1(P^*)
		\\
	&\ge & \frac{1}{n^{1/4}}
	\end{eqnarray*}
This inequality holds with a probability at least $1 - 2^k e^{-2n\left(\frac{L_1(P^*)}{16}\right)^2} = 1 - 2^k e^{-\frac{n}{128} L_1(P^*)^2}$, for any joint distribution $P$ in the open ball $B_{L_1(P^*)/32}\left( P^* \right)$ and for any $n$ large enough to guarantee that $\frac{1}{4} L_1(P^*) \ge \frac{1}{n^{1/4}}$ (by inequality (\ref{eqn-bound}) and lemma \ref{lem-hoeffding}). Also note that this probability lower bound $1 - 2^k e^{-\frac{n}{128} L_1(P^*)^2}$ converges to $1$ as $n$ tends to infinity. So locally uniform convergence holds. This establishes clause 2.
\end{proof}

The above proof actually establishes not just convergence in probability but also almost sure convergence, which follows from two things: the Borel-Cantelli lemma,\footnote
	{
	For a review of Borel-Cantelli lemma, see chapter 14 of \cite{feller1957introduction}.
    }  
and the fact that the error probabilities in question converge to zero quickly enough so that they sum to a finite number. Indeed, under the null hypothesis, the sum of the error probabilities is $\sum_{n=1}^{\infty} 2^k e^{-\frac{\sqrt{n}}{32}} < \infty$. Under the alternative hypothesis, the sum of the error probabilities is $\sum_{n=1}^{\infty} 2^k e^{-\frac{n}{128} L_1(P^*)^2} < \infty$.

\subsection{Construction of Learning Methods}\label{sec:lemma:construction}

Given a finite set $\VVV$ of variables, the following is an algorithm for constructing learning methods that will be shown to witness the existence claim in theorem \ref{the:general1}. 


\begin{description}
\im[Step 1.] Let each conditional independence statement about $\VVV$ be associated with a test of it that achieves the convergence properties established in proposition \ref{pro-test-of-independence}. Combine those tests into a single ``super'' test $T$, which maps each data set $(\xx_1, \ldots, \xx_n)$ to the set $\Sigma = T(\xx_1, \ldots, \xx_n)$ of all the conditional independence statements accepted by their associated tests given data set $(\xx_1, \ldots, \xx_n)$. 

\im[Step 2.] Linearly order all Markov hypotheses about $\VVV$ into a sequence $H_{G_1}, H_{G_2}, \ldots, H_{G_k}$ such that $\II(G_i) \supset \II(G_j)$ implies $i < j$.

\im[Step 3.] Construct a function $F$ that maps each set $\Sigma$ of conditional independence statements about $\VVV$ to the first hypothesis $H_{G_i}$ in the sequence such that $\II(G_i) \subseteq \Sigma$.
\im[Step 4.] Construct learning method $\hat{H} = F \circ T$.
\end{description}  

A graph $G$ on $\VVV$ is said to be {\bf minimal} to a set $\Sigma$ of conditional independence statements if there is no graph $G'$ on $\VVV$ such that $\II(G) \subset \II(G') \subseteq \Sigma$.

\begin{lemma}\label{lem-preservation}
There is a learning method that can be constructed from the above procedure. Furthermore, any such learning method $\hat{H} = F \circ T$ has the following properties:
	\ope 
	\im Whenever $F(\Sigma) = H_G$, then $G$ is minimal to $\Sigma$. 
	\im Whenever $F(\Sigma) = H_G$, then $F(\Sigma') = H_G$ for any set $\Sigma'$ with $\II(G) \subseteq \Sigma' \subseteq \Sigma$.
    \ede 
\end{lemma}

\begin{proof}
The existence of such a learning method follows from the following three facts. First, there exists a ``super'' test $T$ of conditional independence with the property required in step 1 (by proposition \ref{pro-test-of-independence}). Second, there exists a sequence of causal hypotheses with the property required in step 2 (which is obvious because there are only finitely many hypotheses to be ordered). Finally, function $F$ is well-defined (because, as an elementary result in the theory of Bayesian networks, for each set $\Sigma$ of conditional independence statements about $\VVV$, there exists a graph $G$ on $\VVV$ such that $\II(G) = \varnothing \subseteq \Sigma$). 

Consider an arbitrary learning method $\hat{H}$ that can be constructed from the above procedure: $\hat{H} = F \circ T$, with a function $F$, a test $T$, and a sequence of causal hypotheses $H_{G_1}, H_{G_2}, \ldots, H_{G_k}$ satisfying all the required properties. Argue for the two clauses as follows.

To establish clause 1, suppose for {\em reductio} that $F(\Sigma) = H_{G}$ but $G$ is not minimal to $\Sigma$, namely there is a graph $G'$ on $\VVV$ such that $\II(G) \subset \II(G') \subseteq \Sigma$. Since the sequence $H_{G_1}, H_{G_2}, \ldots, H_{G_k}$ contains all the Markov equivalence hypotheses about $\VVV$, we have that $H_G = H_{G_j}$ and $H_{G'} = H_{G_i}$ for some $j, i \le k$. So, to rewrite what we have already had: $F(\Sigma) = H_{G_j}$ and $\II(G_j) \subset \II(G_i) \subseteq \Sigma$. Since $\II(G_i) \supset \II(G_j)$, by the requirement in step 2 of the procedure we have that $i < j$. That is, $H_{G_i}$ is a hypothesis that occurs earlier than $H_{G_j}$ does in the sequence. But note that $\II(G_i) \subseteq \Sigma$. So, by the requirement in step 3, $F(\Sigma)$ is not $H_{G_j}$ but must be either $H_{G_i}$ or some earlier hypothesis in the sequence---contradiction. This establishes clause 1. 

To establish clause 2, suppose that $F(\Sigma) = H_{G}$ and that $\II(G) \subseteq \Sigma' \subseteq \Sigma$. It suffices to show that $F(\Sigma') = H_{G}$. Since the sequence $H_{G_1}, H_{G_2}, \ldots, H_{G_k}$ contains all the Markov equivalence hypotheses about $\VVV$, we have that $F(\Sigma) = H_{G} = H_{G_i}$ and $\II(G) = \II(G_i)$ for some index $i$ of the sequence. Since $F(\Sigma) = H_{G_i}$, by the requirement in step 3 we have:
	\op 
    \im[(i)] $\II(G_{i'}) \not\subseteq \Sigma$ for each $i' < i$.
    \ed 
Since $\II(G) = \II(G_i)$ and $\II(G) \subseteq \Sigma'$ (by hypothesis), we have:
	\op 
    \im[(ii)] $\II(G_i) \subseteq \Sigma'$.
    \ed
Since (i) holds and $\Sigma' \subseteq \Sigma$ (by hypothesis), we have:
    \op 
    \im[(iii)] $\II(G_{i'}) \not\subseteq \Sigma'$ for each $i' < i$,
    \ed 
So, by (ii) and (iii) and the requirement in step 3, we have that $F(\Sigma') = H_{G_i}$.  It follows that $F(\Sigma') = H_G$. This establishes clause 2.
\end{proof}

\begin{lemma} \label{lem-inflexibility}
For every u-minimal causal state $s=(G, P)$ and every learning method $\hat{H} = F \circ T$ that can be constructed from the above procedure, we have that $F(\II(P)) = H_G$.
\end{lemma}
\begin{proof}
Immediate from the requirements in steps 2 and 3.
\end{proof}

The above lemma is the last one we need for proving clause 2 of theorem \ref{the:general1}. The next lemma will be used to prove clause 3 of theorem \ref{the:general2}.

\begin{lemma} \label{lem-flexibility}
For every minimal causal state $s=(G, P)$, there is a learning method $\hat{H} = F \circ T$ that can be constructed from the above procedure such that $F(\II(P)) = H_G$.
\end{lemma}
\begin{proof}
Let $s=(G, P)$ be any minimal causal state, and let $\mathcal{M}(P)$ denote the set of all Markov equivalence hypotheses whose graphs are minimal to $\II(P)$. Since $s$ is minimal, we have: first, $H_G\in \mathcal{M}(P)$; second, for every $H_{G'} \in \mathcal{M}(P)$ distinct from $H_G$, $\II(G') \not\supset \II(G)$. Hence, there is a linear order of all the Markov equivalence hypotheses about $\VVV$, $H_{G_1}, H_{G_2}, \ldots, H_{G_k}$, such that (i) $\II(G_i) \supset \II(G_j)$ implies $i < j$, (ii) $H_G$ = $H_{G_m}$ for some index $m$, and for every $H_{G'} \in \mathcal{M}(P)$ distinct from $H_G$, $H_{G'} = H_{G_n}$ for some index $n$ and $m<n$. Thanks to (i), this linear order can be used in step 2 of the above procedure, which, by clause 1 of lemma~\ref{lem-preservation}, yields a learning method $\hat{H} = F \circ T$ such that $F(\II(P)) \in \mathcal{M}(P)$. Then, because of (ii) and the requirement of step 3 of the procedure, it follows that $F(\II(P)) = H_{G_m}=H_G$.      
\end{proof}

\begin{proposition}\label{pro-key-existential-claim}
Let $\PPP$ be any causal learning problem such that $\VVV$ is a finite set of categorical variables, $\SSS$ is the state space consisting of all causal states on $\VVV$, and $\HH$ is the hypothesis set consisting of all the Markov equivalence hypotheses about $\VVV$. Then there is a learning method for $\PPP$ that can be constructed from the above procedure, and any such learning method has the following properties:
	\op 
	\im[$(a)$] convergence to the truth almost everywhere,
		
	\im[$(b)$] on a maximal domain,
		
	\im[$(c)$] with adherent local uniformity.
	\ed 
\end{proposition}

\begin{proof}
Let $\hat{H}$ be a learning method that can be constructed from the above procedure: $\hat{H} = F \circ T$. 

To prove property $(a)$, note that $\hat{H}$ converges to the truth in every u-minimal state in $\SSS$, thanks to construction step 1, the convergence/consistency property of $T$ established in proposition \ref{pro-test-of-independence}, and lemma \ref{lem-inflexibility}. So $\hat{H}$ fails to converge to the truth {\em only} in states in $\SSS$ that are not u-minimal, but those states form a nowhere dense subset of $\SSS$ (thanks to proposition \ref{pro-unfaithful-nowhere-dense}). So property $(a)$ follows.   

To prove property $(b)$, consider an arbitrary learning method $\hat{H}'$ that converges to the truth in all states where $\hat{H}$ does. It suffices to show that $\hat{H}'$ does not converge to the truth in more states than $\hat{H}$ does. Let $(G, P) \in \SSS$ be a state in which $\hat{H}'$ converges to the truth. It suffices to show that $\hat{H}$ converges to the truth in $(G, P)$. Recall that $\hat{H} = F \circ T$, and by construction step 3, that $F(\II(P)) = H_{G'}$ for some graph $G'$ Markov to $P$. So $(G', P)$ is a state in $\SSS$. Then, by proposition \ref{pro-test-of-independence}, $\hat{H}$ converges to the truth in state $(G', P)$---and, hence, $\hat{H}'$ does, too, by hypothesis. To sum up, $\hat{H}'$ converges to the truth in both states $(G, P)$ and $(G', P)$, which share the same sampling distribution. So it much that $H_G = H_{G'}$. It follows that, since $\hat{H}$ converges to the truth in state $(G', P)$, it also does in state $(G, P)$, as desired. 

To show that property $(c)$ applies to $\hat{H} = F \circ T$, suppose that $\hat{H}$ converges to the truth in a causal state $s = (G, P) \in \SSS$. So $H_G = F(\II(P))$. Then, by lemma~\ref{lem-preservation}, $G$ is minimal to $\II(P)$, so $G$ is minimal to $P$. By proposition \ref{pro-open}, we have:
	\op 
    \im[(i)] State $s$ has an open neighborhood $B_\epsilon(s)$ with a sufficiently small radius $\epsilon$ such that, for any state $s' = (G', P')$ in that open neighborhood $B_\epsilon(s)$, $G' = G$ and $\II(G) \subseteq \II(P') \subseteq \II(P)$.
    \ed 
Now, recall that, by construction step 1, super test $T$ consists of a test $T_\sigma$ for each conditional independence statement $\sigma$ in $\II(\VVV)$ with the convergence properties established in proposition \ref{pro-test-of-independence}. So: 
	\op 
    \im[(ii)] For each conditional independence statement $\sigma_i$ in $\II(G)$, which holds everywhere on open ball $B_{\epsilon}(s)$ by (i), the test $T_{\sigma_i}$ converges to the correct acceptance of $\sigma_i$ uniformly on $B_{\epsilon}(s)$, by clause 1 of proposition \ref{pro-test-of-independence}. 
    
    \im[(iii)] For each conditional independence statement $\sigma_j$ in $\II(\VVV) \smallsetminus \II(P)$, which is violated everywhere on open ball $B_{\epsilon}(s)$ by (i), there exists a radius $\epsilon_j \le \epsilon$ such that test $T_{\sigma_j}$ converges to the correct rejection of $\sigma_j$ uniformly on $B_{\epsilon_j}(s)$, by clause 2 of proposition \ref{pro-test-of-independence}. 
    \ed 
Now, let $\epsilon'$ be the minimum of the radius $\epsilon$ and the radii $\epsilon_j$ constructed in (iii). Then we have:
	\op 
    \im[(iv)] Causal state $s$ has an open neighborhood, namely $B_{\epsilon'}(s) \subseteq \SSS_G$ with $\epsilon' \le \epsilon$, on which the test $T$ of conditional independence converges uniformly to the correct acceptance of all the conditional independence statements in $\II(G)$ (by (ii)) and the correct rejection of all the conditional independence statements in $\II(\VVV) \smallsetminus \II(P)$ (by (iii)). That is,
	$$
	\inf_{s' \in B_{\epsilon'}(s)} \Prob_{s'} \Big( \II(G) \subseteq T(\XX_1, \ldots, \XX_n) \subseteq \II(P) \Big) \to 1  \quad\mathrm{as}\quad n \to \infty \,.
	$$
    \ed 
Since $F(\II(P)) = H_G$, we have: $\II(G) \subseteq T(\xx_1, \ldots, \xx_n) \subseteq \II(P)$ implies $F(T(\xx_1, \ldots, \xx_n)) = H_G$ (by the second clause of lemma~\ref{lem-preservation}). It follows that
	$$
	\inf_{s' \in B_{\epsilon'}(s)} \Prob_{s'} \Big( F(T(\XX_1, \ldots, \XX_n)) = H_G \Big) \to 1  \quad\mathrm{as}\quad n \to \infty \,.
	$$
For every causal state $s'$ in $B_{\epsilon'}(s)$, $s'$ shares the same causal graph $G$ with $s$, so $H_{s'} = H_G$. Therefore,
	$$
	\inf_{s' \in B_{\epsilon'}(s)} \Prob_{s'} \Big( F(T(\XX_1, \ldots, \XX_n)) = H_{s'} \Big) \to 1  \quad\mathrm{as}\quad n \to \infty \,.
	$$
But $\hat{H} = F\circ T$. So,
	$$
	\inf_{s' \in B_{\epsilon'}(s)} \Prob_{s'} \Big( \hat{H}(\XX_1, \ldots, \XX_n) = H_{s'} \Big) \to 1  \quad\mathrm{as}\quad n \to \infty \,,
	$$
which establishes property $(c)$, as desired.
\end{proof}

The existence result (clause 2) of theorem \ref{the:general1} follows immediately from the preceding proposition.

\section{Proof of Theorem \ref{the:general2}}\label{sec:proof-of-theorem}

This appendix is devoted to proving theorem \ref{the:general2}. Clause 1.2 follows immediately from lemmas \ref{lem:dense-open} and \ref{lem-crucial} (in section \ref{sec:proof-crucial-lemma}). So it remains to establish clause 1.1 and clause 2. 

To establish clause 1.1, let $\hat{H}$ be any learning method for causal learning problem $\PPP$ that achieves the joint mode $(a)$+$(b)$+$(c)$. Suppose for {\em reductio} that there is a u-minimal causal state $(G, P)$ in which $\hat{H}$ does not converge to the truth, i.e., does not converge to $H_G$. Consider the learning method $\hat{H}^*$ that rides on $\hat{H}$ as follows.
	\begin{quote}
	\textsc{Definition of} $\hat{H}^*$: Run a super test $T$ of all the conditional independence statements about $\VVV$, with the convergence properties established in proposition~\ref{pro-test-of-independence}. If $T$ accepts exactly the statements in $\II(P)$ (no more and no less), returns $H_G$; otherwise, apply $\hat{H}$.
	\end{quote}
We now show that $\hat{H}^*$ converges to the truth in every causal state in which $\hat{H}$ does. Let $(G', P')$ be any causal state in which $\hat{H}$ converges to the truth. By clause 1.2 (which has been established), $(G', P')$ is a minimal causal state. To show that $\hat{H}^*$ converges to the truth in state $(G', P')$, discuss two exhaustive cases: either $\II(P') = \II(P)$, or not. Case 1: suppose that $\II(P') = \II(P)$. Then $\II(G') = \II(G)$, for three reasons: first, $(G, P)$ is u-minimal; second, $(G', P')$ is minimal; and third, $\II(P') = \II(P)$. Since $\II(G') = \II(G)$, we have that $H_{G'} = H_{G}$. Note that $T$ has the convergence properties established in proposition~\ref{pro-test-of-independence}; so, in state $(G', P')$, $T$ converges to the acceptance of all and only the statements in $\II(P')$, which is identical to $\II(P)$. So $\hat{H}^*$ converges to hypothesis $H_G$ in state $(G', P')$. But $H_G = H_{G'}$. So $\hat{H}^*$ converges to the truth $H_{G'}$ in state $(G', P')$. Now turn to case 2: suppose that $\II(P') \not= \II(P)$. So, in state $(G', P')$, $T$ converges to exactly the statements in $\II(P')$, and hence it is not the case that $T$ converges to exactly the statements in $\II(P)$. So, in state $(G', P')$, $\hat{H}^*$ converges to whatever $\hat{H}$ converges to. With the above discussion of the two exhaustive cases, it follows that $\hat{H}^*$ converges to the truth in every causal state in which $\hat{H}$ does. Moreover, thanks to $T$, $\hat{H}^*$ converges to the truth in state $(G, P)$, in which $\hat{H}$ does not by hypothesis. Therefore, $\hat{H}$ does not achieve convergence to truth on a maximal domain---contradiction. This establishes clause 1.1.

To establish clause 2, consider any causal state $(G_1, P)$ that is minimal but not u-minimal. Since it is not u-minimal, there exists $G_2$ such that $H_{G_1}\neq H_{G_2}$ and $(G_2, P)$ is also a minimal causal state. By lemma~\ref{lem-flexibility} and proposition~\ref{pro-key-existential-claim}, there is a learning method $\hat{H}_1$ that achieves the joint mode $(a)$+$(b)$+$(c)$ and converges to the truth in $(G_1, P)$, and a learning method $\hat{H}_2$ that achieves the joint mode $(a)$+$(b)$+$(c)$ and converges to the truth in $(G_2, P)$. Since $H_{G_1}\neq H_{G_2}$, $\hat{H}_2$ does not converge to the truth in $(G_1, P)$. Clause 2 follows.

\section{An Illustrated Explanation of Why Clause 3 of Theorem \ref{the:general1} Holds}\label{sec:clause3}

Recall the example illustrated in figure \ref{fig-two-spaces}; for ease of reference, it is illustrated below in the simplified figure \ref{fig-violation1}.
	\begin{figure}[ht]
	\centering	\includegraphics[width=0.7\textwidth]{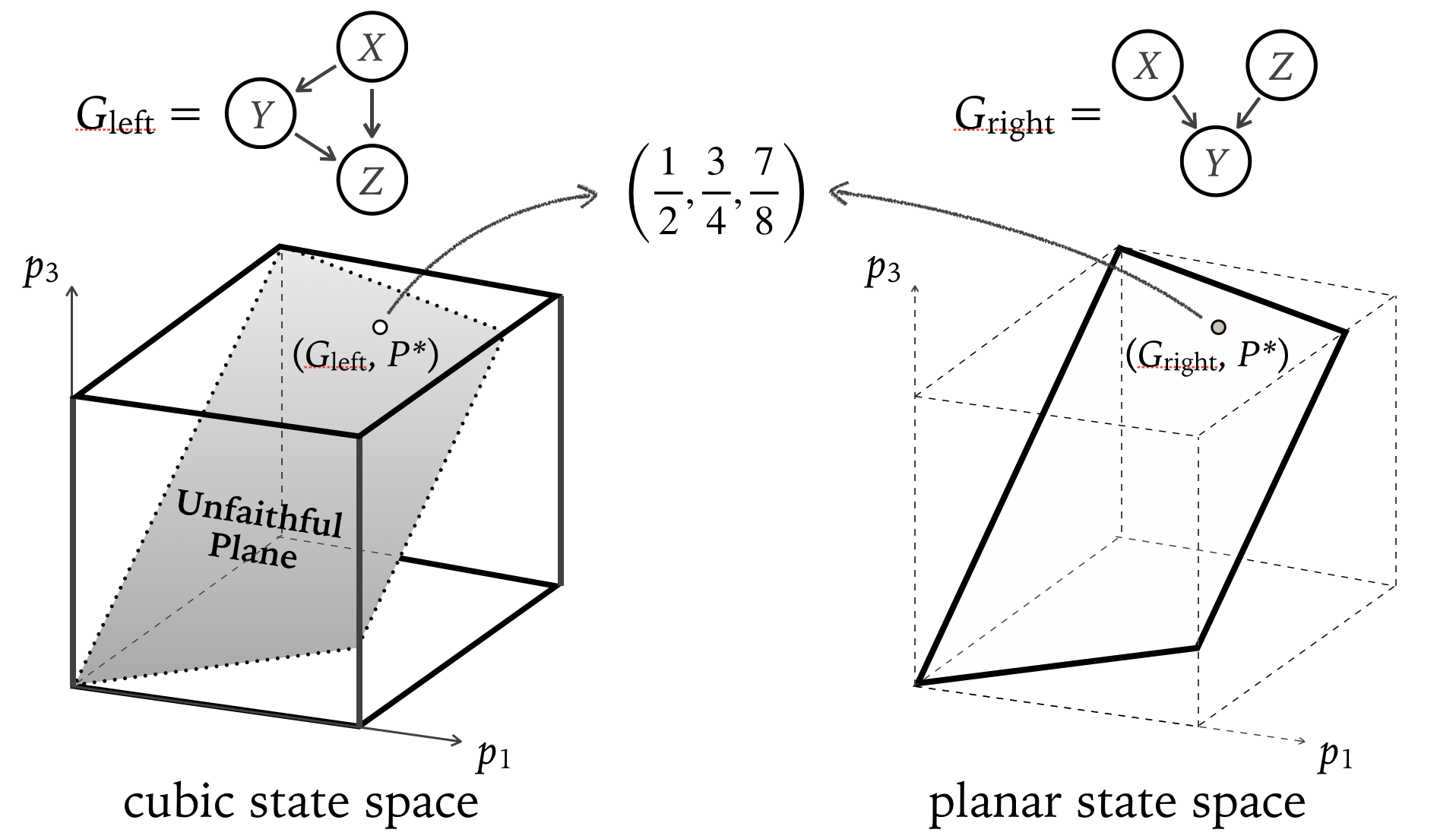}
	\caption{violation of adherently locally uniform convergence}
	\label{fig-violation1}
	\end{figure}
Note that the left, cubic state space embeds a trapezoidal plane, which is an identical copy of the planar state space on the right. The left trapezoid contains all and only the unfaithful causal states in the left cubic state space; so call it the {\bf unfaithful plane}, as indicated in figure \ref{fig-violation1}. Every causal state on the left, embedded trapezoid shares an identical joint distribution with a corresponding causal state on the right, planar state space. For any such pair of causal states, the convergence property has to be sacrificed in at least of the two. The standard design practice would sacrifice the convergence property on the left trapezoid. But consider the alternative proposal that makes sacrifices in accordance with the standard practice except that, for the distribution $P^*$ parametrized by $\left(p_1, p_2, p_3\right) = \left(\frac{1}{2}, \frac{3}{4}, \frac{7}{8}\right)$, sacrifices are made in the right, faithful causal state $(G_\rii, P^*)$ instead of the left, unfaithful causal state $(G_\lee, P^*)$. So, on this alternative proposal, the shaded areas in figure \ref{fig-violation1} are the places where sacrifices are made: a shaded point on the right, together with a shaded, punched plane on the left. On this alternative proposal, the convergence property is secured in the left causal state $(G_\lee, P^*)$ but sacrificed in some causal states that are {\em arbitrarily close} to that causal state, which leads to a violation of adherently locally uniform convergence. 

To avoid such a violation, one might try to secure the convergence property not just in the left causal state $(G_\lee, P^*)$ but in all of its nearby states, as depicted by the open disc on the left side of figure \ref{fig-violation2}. (The shaded areas are still understood as the places where sacrifices are made.)
	\begin{figure}[ht]
	\centering	\includegraphics[width=0.7\textwidth]{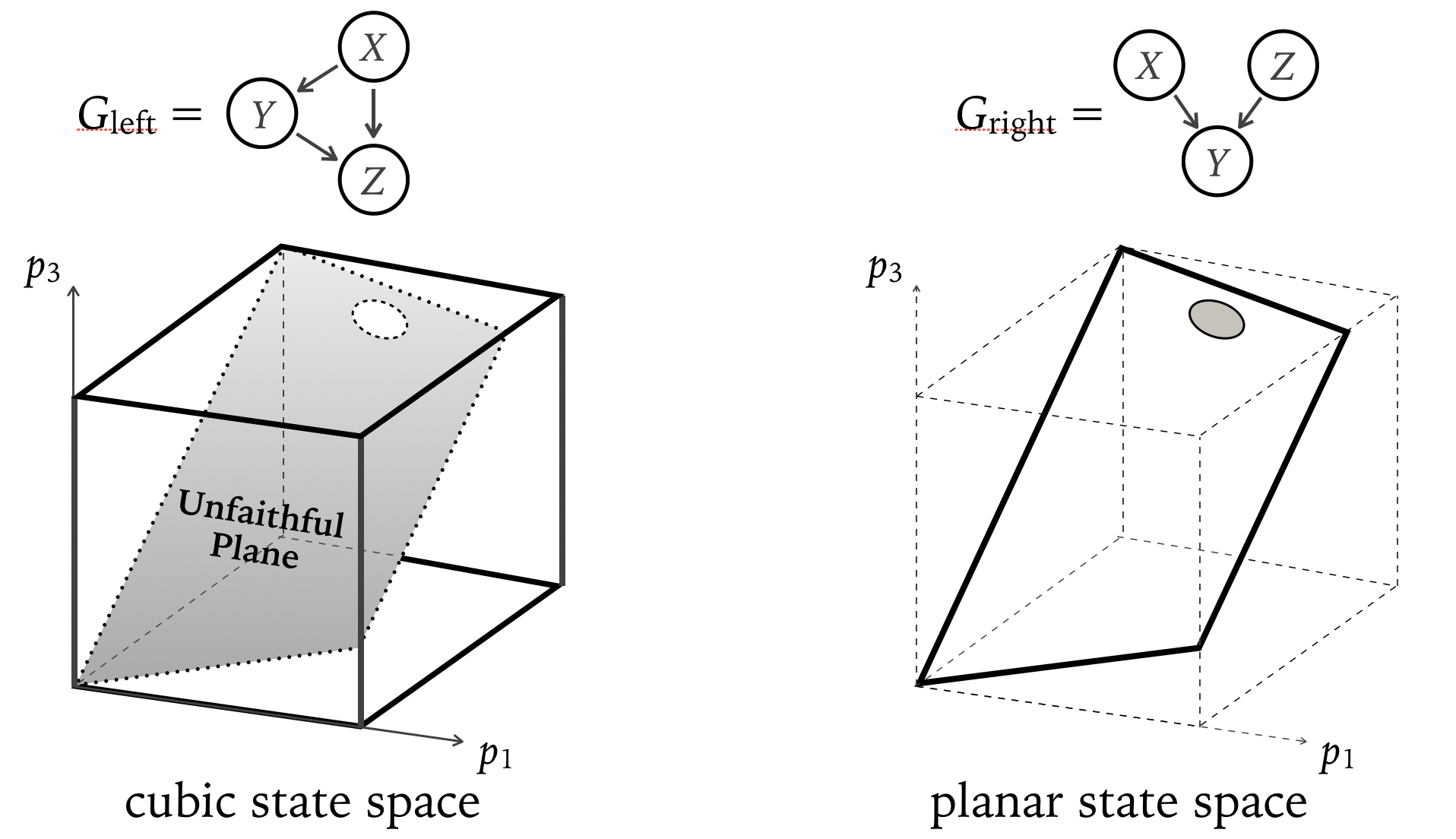}
	\caption{violation of almost everywhere convergence}
	\label{fig-violation2}
	\end{figure}
But doing so would force the convergence property to be sacrificed on the corresponding disc on the right, planar state space, which leads to a violation of almost everywhere convergence. 

So, there is only one way to avoid both the two kinds of violations depicted in figures \ref{fig-violation1} and  \ref{fig-violation2}: given any causal state $(G_\lee, P)$ on the left, unfaithful plane and the corresponding causal state $(G_\rii, P)$ in the right, planar state space, the convergence property has to be sacrificed in the left, unfaithful one. 

This allows for the possibility of converging to the truth in any faithful causal state. And this possibility can be forced into a reality by requiring a maximal domain of convergence.

\section{Some More Details of the Example in Section \ref{sec:setup1}}\label{sec:example}

Given the background assumption of the example in section \ref{sec:setup1}, there are two possible causal structures on the table with three parameters $p_1, p_2$, and $p_3$ whose values are unknown. For each of those two causal structures, the conditional probability of every effect given its immediate causes can be expressed by the three parameters, as indicated in figure \ref{fig-probabilities}.

	\begin{figure}[ht]
	\centering	\includegraphics[width=1.0\textwidth]{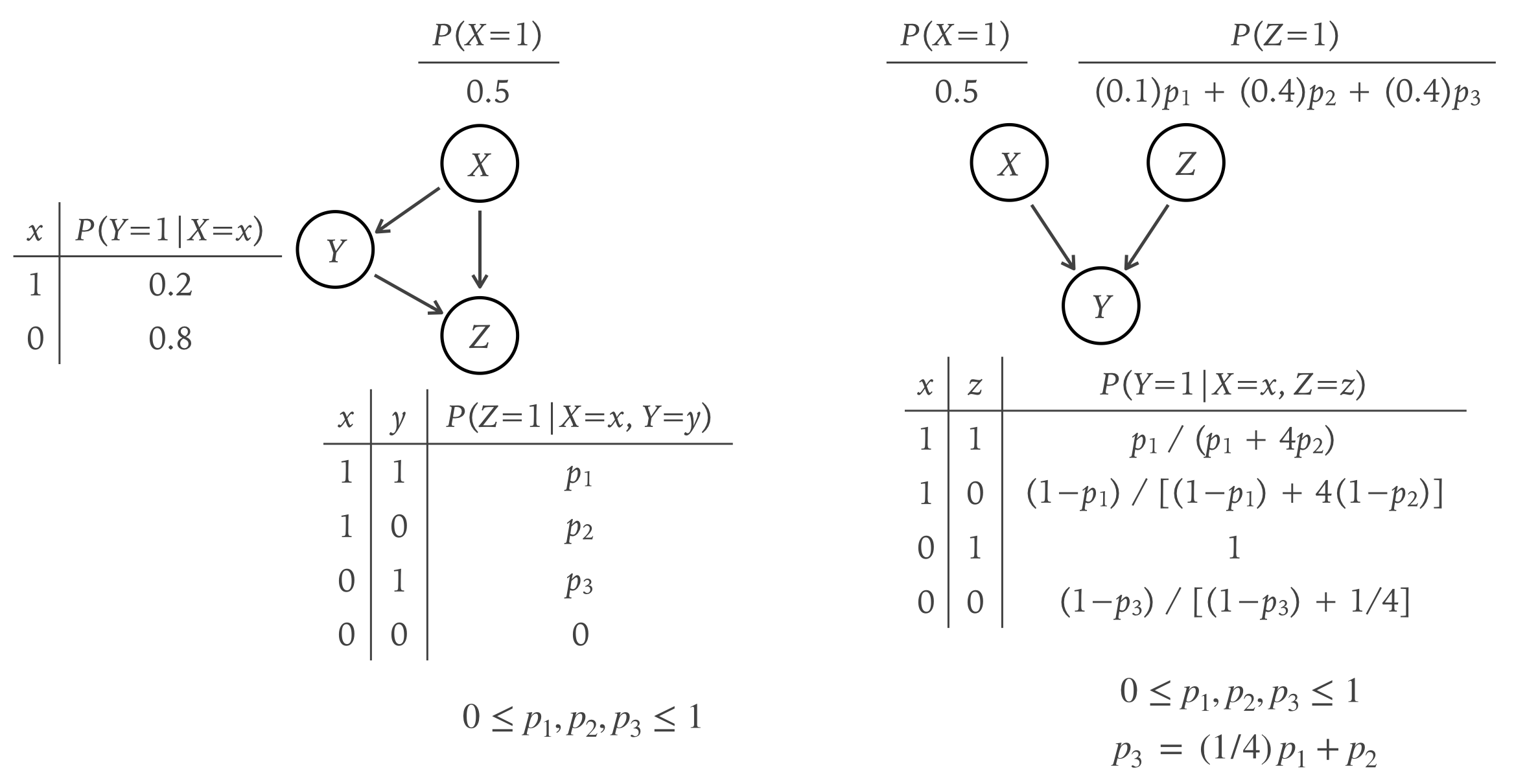}
	\caption{Conditional probability tables}
	\label{fig-probabilities}
	\end{figure}

When it is said that the same joint distribution is shared, what is actually meant is only that the same joint distribution is shared {\em in the absence of manipulation}. To illustrate, consider the two CBNs $\big(G_\lee, P^* \big)$ and $\big(G_\rii, P^* \big)$ that share the same joint distribution $P^*$ parametrized by $(p_1, p_2, p_3) = (\frac{1}{2}, \frac{3}{4}, \frac{7}{8})$. Also consider the manipulation that forces $Y = 0$. If the true CBN is the right one, the manipulation $Y = 0$ is only a manipulation of an effect rather than a cause (see the right causal graph $G_\rii$); so the distribution of $Z$ would remain the same were this manipulation applied---in particular, the probability of $Z=0$ would remain at $30\%$. But the same manipulation would raise the probability of $Z=0$ from $30\%$ to $62.5\%$ if instead the true CBN is the one on the left $\big(G_\lee, (\frac{1}{2}, \frac{3}{4}, \frac{7}{8}) \big)$. Indeed, in this case, the manipulation $Y = 0$ is a manipulation of a cause of $Z$. Therefore, whether the true CBN is the one on the left or on the right makes an important difference---at least for those who are thinking about manipulating $Y$ in order to change $Z$. So it would be great if there exists a learning method that can distinguish between those two CBNs. Unfortunately, there exists no such learning method if the available data are non-experimental (i.e., collected without any manipulation of the true, unknown CBN), as explained in section \ref{sec:setup2}.

\end{document}